%% file: main.tex
\documentclass{article}

\usepackage[main,final]{neurips_2026}

\usepackage[utf8]{inputenc} %
\usepackage[T1]{fontenc}    %
\usepackage{url}            %
\usepackage{booktabs}       %
\usepackage{amsfonts}       %
\usepackage{nicefrac}       %
\usepackage{microtype}      %
\usepackage{xcolor}         %

\definecolor{DarkRed}{rgb}{0.368,0.097,0.078}
\definecolor{DarkBlue}{rgb}{0,0,0.5}
\usepackage[
    colorlinks=true,
    linkcolor=DarkBlue,
    citecolor=DarkBlue,
    urlcolor=DarkRed,
    pagebackref
]{hyperref}

\usepackage{bm}
\usepackage{mathtools}
\usepackage{amssymb}
\usepackage{amsthm}
\usepackage{thmtools}
\usepackage{thm-restate}

\usepackage{prettyref}

\usepackage{hyperref}

\usepackage[nameinlink]{cleveref}
\input{command.tex}

\allowdisplaybreaks

\RestyleAlgo{vlined}
\Crefname{app}{Appendix}{Appendices}

\title{Strongly Adaptive Online Learning \\with Time-Varying Movement Costs}

\author{%
Andrew Jacobsen\\
Universit\`{a} degli Studi di Milano\\
Politecnico di Milano\\
\texttt{contact@andrew-jacobsen.com}
\And
Emmanuel Esposito\\
Universit\`{a} degli Studi di Milano\\
\texttt{emmanuel@emmanuelesposito.it}
\And
Hao Qiu\\
National University of Singapore\\
\texttt{qiuhaosai@gmail.com}
\And
Mengxiao Zhang\\
University of Iowa\\
\texttt{mengxiao-zhang@uiowa.edu}
}

\begin{document}

\maketitle

\begin{abstract}
We study online convex optimization (OCO) with time-varying movement costs, where the learner incurs both a loss and a time-dependent penalty for changing its decisions. While previous work has established dynamic regret guarantees for this setting, obtaining strongly adaptive guarantees that hold simultaneously over arbitrary intervals remains an open challenge. We address this by developing a new algorithm which, for every interval $[a,b]\subseteq[T]$ and arbitrary comparator sequence, achieves a data-dependent regret bound of $ \widetilde{\mathcal{O}}\left( \sqrt{(D^2+DP_{[a,b]})\sum_{t \in [a,b]}(\|g_t\|^2+\lambda_t\|g_t\|)} \right) $, where $D$ is the diameter of the decision domain, $P_{[a,b]}$ is the path length of the comparator sequence over $[a,b]$, $g_t$ is the gradient of the loss, and $\lambda_t$ is the movement-cost coefficient at round~$t$. Key to our approach is a novel reduction to sleeping experts with movement costs, a previously unexplored variant that is of independent interest.  As a result of our development, we further obtain strongly-adaptive dynamic regret guarantees under arbitrary delayed feedback with a refined dependence on interval-restricted cumulative delays. Our approach also naturally extends to OCO with time-varying memory, where we obtain  strongly-adaptive guarantees which improve the best-known bounds.
\end{abstract}

\input{sections/introduction}

\input{sections/setting}

\input{sections/movement_cost}

\input{sections/sleeping}

\input{sections/main-result}

\input{sections/applications}

\input{sections/discussion}

\newpage

\begin{ack}
EE and AJ acknowledge the EU Horizon CL4-2022HUMAN-02 research and innovation action under grant agreement 101120237, project ELIAS (European Lighthouse of AI for Sustainability). HQ is supported by the Singapore National Research Foundation (NRF) under its AI Visiting Professorship programme.
\end{ack}

\DeclareRobustCommand{\VAN}[3]{#3} 
\bibliography{refs}
\bibliographystyle{abbrvnat}
\DeclareRobustCommand{\VAN}[3]{#2}

\newpage
\appendix
\crefalias{section}{appendix}
\crefalias{subsection}{appendix}
\crefalias{subsubsection}{appendix}

\input{sections/more_related}

\input{sections/appendix-samc}
\input{sections/appendix-sleeping}

\input{sections/appendix-pfmc}

\input{sections/appendix-adaptive-gd}

\input{sections/appendix-first-order-redux}

\input{sections/appendix-main}
\input{sections/appendix-applications}

\input{sections/appendix-lemmas}

\newpage
\input{checklist.tex}

\end{document}

%% file: command.tex
\newcommand{\tp}{\tilde{p}}

\def \R {\mathbb{R}}

\newcommand{\calA}{{\mathcal{A}}}

\newcommand{\calS}{{\mathcal{S}}}

\newcommand{\calW}{{\mathcal{W}}}

\newcommand{\cF}{{\mathcal{F}}}
\newcommand{\cA}{{\mathcal{A}}}

\newcommand{\cE}{{\mathcal{E}}}
\newcommand{\cI}{{\mathcal{I}}}

\newcommand{\cO}{{\mathcal{O}}}

\newcommand{\cT}{{\mathcal{T}}}

\newcommand{\cP}{\mathcal{P}}

\newcommand{\gtilde}{\Tilde{g}}
\newcommand{\wt}[1]{\widetilde{#1}}
\newcommand{\inner}[1]{ \left\langle {#1} \right\rangle }

\newcommand{\order}{\mathcal{O}}
\usepackage[vlined,algo2e]{algorithm2e}
\usepackage{algorithm}

\newcommand{\pref}[1]{\prettyref{#1}}

\newcommand{\savehyperref}[2]{\texorpdfstring{\hyperref[#1]{#2}}{#2}}
\newrefformat{eq}{\savehyperref{#1}{Eq. \textup{(\ref*{#1})}}}
\newrefformat{eqn}{\savehyperref{#1}{Eq.~(\ref*{#1})}}
\newrefformat{lem}{\savehyperref{#1}{Lemma~\ref*{#1}}}
\newrefformat{def}{\savehyperref{#1}{Definition~\ref*{#1}}}
\newrefformat{line}{\savehyperref{#1}{Line~\ref*{#1}}}
\newrefformat{thm}{\savehyperref{#1}{Theorem~\ref*{#1}}}
\newrefformat{corr}{\savehyperref{#1}{Corollary~\ref*{#1}}}
\newrefformat{cor}{\savehyperref{#1}{Corollary~\ref*{#1}}}
\newrefformat{sec}{\savehyperref{#1}{Section~\ref*{#1}}}
\newrefformat{app}{\savehyperref{#1}{Appendix~\ref*{#1}}}
\newrefformat{assum}{\savehyperref{#1}{Assumption~\ref*{#1}}}
\newrefformat{asm}{\savehyperref{#1}{Assumption~\ref*{#1}}}
\newrefformat{tab}{\savehyperref{#1}{Table~\ref*{#1}}}
\newrefformat{ex}{\savehyperref{#1}{Example~\ref*{#1}}}
\newrefformat{fig}{\savehyperref{#1}{Figure~\ref*{#1}}}
\newrefformat{alg}{\savehyperref{#1}{Algorithm~\ref*{#1}}}
\newrefformat{rem}{\savehyperref{#1}{Remark~\ref*{#1}}}
\newrefformat{conj}{\savehyperref{#1}{Conjecture~\ref*{#1}}}
\newrefformat{prop}{\savehyperref{#1}{Proposition~\ref*{#1}}}
\newrefformat{proto}{\savehyperref{#1}{Protocol~\ref*{#1}}}
\newrefformat{prob}{\savehyperref{#1}{Problem~\ref*{#1}}}
\newrefformat{event}{\savehyperref{#1}{Event~\ref*{#1}}}
\newrefformat{claim}{\savehyperref{#1}{Claim~\ref*{#1}}}
\newrefformat{ineq}{\savehyperref{#1}{Eq.~\ref*{#1}}}
\newrefformat{que}{\savehyperref{#1}{Question~\ref*{#1}}}
\newrefformat{op}{\savehyperref{#1}{Open Problem~\ref*{#1}}}
\newrefformat{fn}{\savehyperref{#1}{Footnote~\ref*{#1}}}

\DeclarePairedDelimiter{\abs}{|}{|}
\DeclarePairedDelimiter{\norm}{\|}{\|}

\newcommand{\lrb}[1]{\left(#1\right)}
\newcommand{\brb}[1]{\bigl(#1\bigr)}
\newcommand{\Brb}[1]{\Bigl(#1\Bigr)}

\newcommand{\lsb}[1]{\left[#1\right]}
\newcommand{\bsb}[1]{\bigl[#1\bigr]}

\newcommand{\lcb}[1]{\left\{#1\right\}}
\newcommand{\bcb}[1]{\bigl\{#1\bigr\}}

\DeclareMathOperator*{\argmin}{arg\,min}

\crefname{assumption}{assumption}{assumptions}
\Crefname{assumption}{Assumption}{Assumptions}

\newcommand{\dmax}{\ensuremath{d_{\mathrm{max}}}}
\newcommand{\dtot}{\ensuremath{d_{\mathrm{tot}}}}
\newcommand{\msize}{\ensuremath{b}}
\newcommand{\bmax}{\ensuremath{\msize_{\mathrm{max}}}}

\newcommand{\ww}{\mathcal{W}}
\newcommand{\cmp}{u}
\newcommand{\tpp}{{t+1}}
\newcommand{\tmm}{{t-1}}
\newcommand{\w}{w}

\newcommand{\wtpp}{\w_{\tpp}}
\newcommand{\wtmm}{\w_{\tmm}}
\newcommand{\sumtT}{\sum_{t=1}^T}
\newcommand{\gt}{g_{t}}
\newcommand{\Log}[1]{\log\left(#1\right)}
\newcommand{\grad}{\nabla}
\newcommand{\brac}[1]{\left(#1\right)}
\newcommand{\zeros}{\boldsymbol{0}}
\newcommand{\sbrac}[1]{\left[#1\right]}
\renewcommand{\wt}{\w_t}
\newcommand{\half}{\frac{1}{2}}
\newcommand{\Set}[1]{\left\{#1\right\}}
\newcommand{\minOp}{\wedge}
\newcommand{\maxOp}{\vee}
\newcommand{\cW}{\ww}

\newcommand{\pt}{p_{t}}
\newcommand{\pti}{p_{t,i}}
\newcommand{\ptj}{p_{t,j}}

\newcommand{\mt}{m_{t}}

\newcommand{\ptmm}{p_{\tmm}}

\newcommand{\meta}{\mathrm{Meta}}

\renewcommand{\tilde}{\widetilde}
\renewcommand{\hat}{\widehat}

\newcommand{\diff}{\mathop{}\!\mathrm{d}}

\newcommand{\defeq}{\triangleq}

\newcommand{\del}{\mathrm{del}}
\newcommand{\mem}{\mathrm{mem}}

\newcommand{\RM}{R^{\mem}}
\newcommand{\RD}{R^{\del}}

\newcommand{\elltilde}{\tilde{\ell}}

\newcommand{\hellt}{\widehat{\ell}_t}
\newcommand{\hellti}{\widehat{\ell}_\ti}

\newcommand{\z}{z}
\newcommand{\zt}{\z_t}

\newcommand{\zti}{z_{ti}}

\newcommand{\tgt}{\tilde{g}_{t}}

\newcommand{\ptmmj}{p_{\tmm,j}}

\newcommand{\gtau}{g_{\tau}}
\newcommand{\ztmm}{z_{\tmm}}
\newcommand{\ztmmi}{z_{\tmm,i}}
\newcommand{\ft}{f_{t}}
\newcommand{\sumab}{\sum_{t\in[a,b]}}

\newcommand{\Ohat}{\hat{\cO}}

\newcommand{\ab}{{[a,b]}}

\newcommand{\Active}{\mathrm{Active}}
\newcommand{\gs}{g_{s}}

\newcommand{\talpha}{\tilde{\alpha}}
\newcommand{\abp}{{[a',b']}}

\newcommand{\tptau}{\tilde{p}_{\tau}}
\newcommand{\tptauI}{\tilde{p}_{\tau,I}}
\newcommand{\telltauI}{\tilde{\ell}_{\tau,I}}
\newcommand{\tptaumm}{\tilde{p}_{\taumm}}
\newcommand{\twtaumm}{\tw_{\taumm}}

\newcommand{\taumm}{{\tau-1}}
\newcommand{\ellti}{\ell_\ti}

  \newcommand{\tell}{\tilde{\ell}}

  \newcommand{\si}{{si}}
  
  \newcommand{\taupp}{{\tau+1}}
  \newcommand{\ztppi}{z_{\tpp,i}}
  \newcommand{\tcmp}{\tilde{\cmp}}
  \newcommand{\telltau}{\tell_{\tau}}
\newcommand{\tg}{\tilde{g}}
\newcommand{\tgtau}{\tg_{\tau}}
\newcommand{\tw}{\tilde{w}}
\newcommand{\twtau}{\tw_{\tau}}

\newcommand{\KL}{\mathrm{KL}}
\newcommand{\tppi}{{\tpp,i}}
\newcommand{\tlambda}{\widetilde{\lambda}}
\newcommand{\ti}{{t,i}}
\theoremstyle{plain}
\newtheorem{theorem}{Theorem}[section]

\newtheorem{lemma}[theorem]{Lemma}

\theoremstyle{definition}

\newtheorem{assumption}[theorem]{Assumption}
\theoremstyle{remark}
\newtheorem{remark}[theorem]{Remark}

\newcommand{\tauI}{{\tau,I}}

\newcommand{\tppI}{{\tpp,I}}
\newcommand{\tmmi}{{\tmm,i}}

\newtheoremstyle{manualtheoremstyle}
  {3pt}      %
  {3pt}      %
  {\itshape} %
  {}         %
  {\normalfont\bfseries} %
  {.}        %
  {.5em}     %
  {}         %

\theoremstyle{manualtheoremstyle}
\newtheorem{manualtheoreminner}{Theorem}
\newenvironment{manualtheorem}[1]{%
  \renewcommand\themanualtheoreminner{#1}%
  \manualtheoreminner
}{\endmanualtheoreminner}

\newtheoremstyle{manuallemmastyle}
  {3pt}      %
  {3pt}      %
  {\itshape} %
  {}         %
  {\normalfont\bfseries} %
  {.}        %
  {.5em}     %
  {}         %

\theoremstyle{manuallemmastyle}
\newtheorem{manuallemmainner}{Lemma}
\newenvironment{manuallemma}[1]{%
  \renewcommand\themanuallemmainner{#1}%
  \manuallemmainner
}{\endmanuallemmainner}

\newcommand{\Indicator}{\boldsymbol{1}}

\usepackage{enumitem}

\newcommand{\SecMain}{Main Result}
\newcommand{\SecStronglyAdaptive}{Strong Adaptivity with Time-Varying Movement Costs}
\newcommand{\SecApplications}{Applications}
\newcommand{\SecSleeping}{Sleeping Experts with Movement Costs}
\newcommand{\atob}{{a:b}}

%% file: sections/introduction.tex
\section{Introduction}

Online convex optimization (OCO) \citep{cesa2006prediction, hazan2016introduction, orabona2019modern} is a fundamental framework in machine learning for modeling and analyzing sequential prediction and decision-making problems. 
Formally, consider $T$ rounds of interaction between an algorithm and the environment.
On each round $t\in[T]$, the learner chooses $\wt$ from a convex feasible domain $\cW\subseteq\R^d$, 
after which the environment reveals a loss function $\ft:\ww\to\R$, and the learner incurs a loss
of $\ft(\wt)$.
The objective is to design algorithms whose cumulative loss is close to that of a benchmark comparator sequence chosen in hindsight.
Classic (static) regret guarantees compare the learner's performance against a single fixed comparator over the entire time horizon, and are therefore most appropriate in stationary environments.
However, many modern applications are inherently non-stationary; 
indeed, in many realistic streaming settings, 
unpredictable changes in the environment can induce shifts in the underlying 
data-generating distribution, rendering regret relative to a \emph{fixed}
benchmark largely uninformative.

Several stronger regret notions have been developed to capture such non-stationarity settings.
Dynamic regret \citep{zinkevich2003online, zhang2018adaptive} takes a global perspective by comparing the learner against a time-varying comparator sequence over the entire horizon, 
$
    R_T(\cmp_{1:T})\defeq\sumtT \ft(\wt)-\ft(\cmp_t). 
$
This allows the benchmark to change over time and provides a natural way to quantify performance in drifting environments.
By contrast, strongly-adaptive regret \citep{daniely2015strongly} requires the learner to achieve low static regret simultaneously on every contiguous time interval.
Rather than focusing only on aggregated performance, strongly adaptive regret measures how well the learner adapts locally to changes at arbitrary time scales:
\(
    R_\ab(\cmp)\defeq\sum_{t\in\ab}\ft(\wt)-\ft(\cmp).
\)
More generally, \citet{zhang2020minimizing} extend this definition to study \emph{simultaneous} dynamic regret and strongly-adaptive regret minimization, and define
\(
    R_\ab(\cmp_{a:b})\defeq\sum_{t\in\ab}\ft(\wt)-\ft(\cmp_t),
\)
which captures both strongly-adaptive regret and dynamic regret as special cases. 

Another crucial challenge in OCO is that changing decisions can itself be costly.
For example, in optimal control  \citep{goel2019online}, video streaming \citep{joseph2012jointly}, and geographical load balancing \citep{lin2012online}, frequent changes may incur switching costs, transaction costs, or instability penalties, and these costs themselves may vary in time.
This motivates OCO with time-varying movement costs, where the learner is charged not only the loss but also a time-varying movement penalty associated with the decision.
This raises the following question:
\begin{center}
    \textit{Can we design algorithms for OCO with \textbf{time-varying movement costs} that achieve \textbf{strongly adaptive} regret against arbitrary time-varying comparator sequences?}
\end{center}

\paragraph{Contributions.}
In this paper, we answer this question affirmatively by developing the first strongly-adaptive algorithms for OCO with time-varying movement costs. Our main technical ingredient is a reduction that decomposes the regret of a mixture predictor into two components: the movement-cost regret of a base expert and the movement-cost regret of a meta-algorithm for the sleeping experts problem. This reduction allows us to combine interval-based specialists with a new algorithm for sleeping experts with movement costs, yielding strongly adaptive regret bounds of the form
$ \widetilde{\mathcal{O}}\Bigl( \sqrt{(D^2+DP_{[a,b]})\sum_{t \in [a,b]}(\|g_t\|^2+\lambda_t\|g_t\|)}\Bigr) $
simultaneously over every interval $[a,b]$, where $D$ is the diameter of the feasible domain, $P_{[a,b]}$ is the drift of the comparator within interval $[a,b]$, and $g_t$ and $\lambda_t$ are the gradient and the movement cost coefficient respectively at round $t$.
Our result builds upon the work of \citet{qiu2026parameterfree}, who establish data-dependent dynamic regret guarantees for OCO with time-varying movement costs over the entire horizon, by extending such guarantees to hold simultaneously over every sub-interval.
Moreover, even when restricted to fixed movement costs and static comparators, our guarantees improve upon the best-known strongly adaptive bounds of \citet{zhang2022adversarial} (see \Cref{sec:main}).

We further extend the delays-to-movement-costs reduction of \citet{qiu2026parameterfree}, which interprets outstanding gradients as time-varying movement costs, to obtain strongly-adaptive dynamic regret guarantees under arbitrary delayed feedback. Specifically, for every interval $\ab$, we achieve a bound of $\widetilde{\order}\Bigl(G\sqrt{(D^2+DP_\ab)\brb{|b-a|+\dtot^{\ab}}}\Bigr)$, up to lower-order boundary terms, where $\dtot^{\ab}$ is the interval-restricted cumulative delay. Unlike the strongly-adaptive guarantees of \citet{wan2025black}, which consider static comparators and depend on the global maximum delay, our bound accommodates arbitrary time-varying comparators on every interval, with an optimal dependence on interval-restricted cumulative delay in its leading term.

We also extend our reduction to online learning with time-varying memory, showing that
memory-dependent losses can be handled through the same movement-cost perspective even when memory lengths change arbitrarily over time.
In particular, we obtain a strongly adaptive regret bound of
$\tilde\cO\Big(\sqrt{(D^2+DP_\ab)\big(H^{2}|b-a| + HG\sum_{t\in\ab}\brb{\msize_t^\ab}^2\Big)}\Big)$
where $\msize_{t}^{\ab}=\min\Set{\msize_{t},t-a}$ is the interval-clipped memory length, $G$ is the
coordinate-wise Lipschitz constant, and $H$ bounds the gradient norm of the unary losses.
Our guarantees extend the existing analysis of OCO with memory \citep{zhang2022adversarial,zhao2023non,qiu2026parameterfree} to arbitrary time-varying memory lengths and dynamic comparator sequences on every interval. Moreover, the resulting bounds adapt to the realized memory lengths within each interval, yielding refined data-dependent guarantees that recover or improve the best-known dependencies in several important special cases.

\paragraph{Related work.}
Existing work on OCO with movement costs has established complementary forms of adaptivity, but has not fully addressed their interaction. \citet{zhang2022adversarial} study strongly adaptive regret with fixed movement costs, obtaining guarantees against static comparators on arbitrary intervals. In contrast, \citet{zhang2021revisiting} and \citet{mhaisen2026partially} study dynamic regret with fixed movement costs, while \citet{qiu2026parameterfree} extend this setting to arbitrary time-varying movement-cost coefficients and establish data-dependent guarantees. However, these dynamic regret guarantees concern the \emph{full time horizon} and do not directly control performance on arbitrary sub-intervals. Our work bridges these two lines of research by establishing strongly adaptive guarantees that simultaneously accommodate time-varying movement costs and dynamic comparators, while preserving a sharp, data-dependent scaling on the realized movement-cost coefficients.

Beyond movement costs, delayed feedback presents another major challenge in non-stationary OCO. \citet{wan2023non} establish dynamic regret guarantees under arbitrary delays. More closely related to our setting, \citet{wan2025black} extends the black-box reduction of \citet{joulani2013online} to obtain strongly-adaptive regret guarantees under arbitrary delays, with a dependence on the maximum delay. However, these guarantees compete with static comparators on each interval, whereas their dynamic regret guarantees concern the entire time horizon. Our work strengthens these results by simultaneously accommodating dynamic comparator sequences on arbitrary intervals and obtaining a refined leading dependence on interval-restricted cumulative delays. Subsequently, \citet{qiu2026parameterfree} proposed a reduction from delayed feedback to OCO with time-varying movement costs, and a related idea was developed by \citet{ryabchenko2026reduction} in the context of bandit feedback. 
Relatedly, OCO with movement costs also extends to OCO with memory, a setting where prior works \citep{zhao2023non,qiu2026parameterfree} have again focused on dynamic regret for both fixed and time-varying memory lengths. Due to space limitations, we defer additional related work to \Cref{app:more_related}.

\paragraph{Notations.}
Throughout the paper, we use $\mathbb{Z}_{+}$ to denote the set of positive integers and $[m]\triangleq\{1,\ldots,m\}$ for any $m\in\mathbb{Z}_{+}$.
By a slight abuse of notation, we often refer to a time interval between rounds $a,b \in \mathbb{Z}_{+}$ with $a \le b$ as $[a,b] = \{a,\dots,b\}$, which contains only the integer values between $a$ and $b$.
We denote the subdifferential set of a convex function $f:\ww\to\R$ at some point $w \in \ww$ as $\partial f(w)$. We write $\mathbf{0}$ for the all-zero vector in the appropriate dimension, $\Indicator\Set{x\in A}$ denotes the indicator function for a set $A$, and we use $\|\cdot\|$ to denote the Euclidean norm unless otherwise specified. The $\ell_1$ norm is denoted $\|\cdot\|_1$.
For two vectors $u,v\in \R^d$, denote the coordinate-wise product as $u\odot v$, whose $i$-th coordinate is $u_iv_i$. For a finite set $\calS$, we use $|\calS|$ to denote the cardinality of $\calS$. The Bregman divergence with respect to a differentiable function $\psi: \calW \to \R$ is $D_\psi(x \mid y) \defeq \psi(x)-\psi(y)- \langle\nabla \psi(y), x-y\rangle$. Define $\Delta_{N-1}$ as the $N$-dimensional probability simplex, and we denote $\Delta_\calS$ for the probability simplex over a finite set $\calS$. We use $\widetilde{\order}(\cdot)$ to suppress all $\log(\cdot)$ terms and use $\widehat{\order}(\cdot)$ to suppress only constant and $\log\log(\cdot)$ terms.

%% file: sections/setting.tex
\section{Problem Setting}%
\label{sec:setting}

We study online convex optimization with time-varying movement costs over a nonempty, closed, convex decision set $\ww\subseteq\mathbb R^d$. We assume throughout that the domain is bounded with diameter
$D=\sup_{x,y\in\ww}\norm{x-y}$, and we assume without loss of generality that $\zeros\in\ww$.
The interaction between the learner and the environment lasts for $T\in\mathbb{Z}_{+}$ rounds. 
On each round $t\in[T]$, the learner chooses a decision
$w_t\in\mathcal W$, then the environment selects a convex loss function
$f_t:\mathcal W\to\mathbb R$, and the learner incurs a cost of
$
    f_t(w_t)+\lambda_t\|w_t-w_{t-1}\|.
$
At the end of the same round $t$, the learner observes a subgradient
$g_t\in\partial f_t(w_t)$ and the next movement-cost coefficient
$\lambda_{t+1}$.
For every (nonempty) time interval $\ab \subseteq [T]$, we measure performance by the movement-cost dynamic regret:
\[
    R_{\ab}(u_{a:b},\lambda_{a:b})
    \triangleq
    \sum_{t=a}^b
    \left(
        f_t(w_t)-f_t(u_t)\right)
        +
        \sum_{t=a+1}^b
        \lambda_t\|w_t-w_{t-1}\|
    \;,
\]
where $u_{a:b}=(u_a,\ldots,u_b)$ is any comparator sequence in  $\ww$. For ease of notation, when the interval is clear from context we will typically define $\lambda_t=0$ for $t\notin[a+1,b]$ so that we may write $\sumab\lambda_t\norm{\wt-\wtmm}$, or simply $\sum_t\lambda_t\norm{\wt-\wtmm}$. 
Similarly, for notational convenience we define pre-interval iterates $w_{a-1}=w_a$ and $u_{a-1}=u_a$ wherever relevant.

For a comparator sequence $u_{a:b}$, we define its path length over $\ab$ as
$
    P_{\ab}
    \triangleq
    \sum_{t=a+1}^{b}\|u_t-u_{t-1}\|.
$ For any time interval $\ab$, define
\(
    \lambda_{\max}^{\ab} \triangleq \max_{t\in\ab}\lambda_t,
\) as the largest movement-cost coefficient over $\ab$. 
This one-sided movement-cost regret follows the stronger objective considered
in prior works on OCO with movement costs
\citep{zhao2023non}. In contrast, a symmetric notion of
dynamic regret would also charge the comparator for its movement. More
precisely, if we define
\begin{align*}
        \mathcal R_{\ab}(u_{a:b},\lambda_{a:b})
    \triangleq
    \sumab\Bigl[\bigl(\ft(\wt)+\lambda_t\norm{\wt-\wtmm}\bigr) - \bigl(\ft(\cmp_t) +\lambda_t\norm{\cmp_t-\cmp_\tmm}\bigr)\Bigr] \;,
\end{align*}
Since the movement costs are nonnegative, the symmetric regret is always upper bounded by its one-sided counterpart:
\(
    \mathcal{R}_{\ab}(\cmp_{a:b},\lambda_{a:b})
    \le
    R_{\ab}(u_{a:b},\lambda_{a:b}).
\) Therefore, any upper bound on the one-sided regret immediately implies the same guarantee for the symmetric regret, so we focus on the stronger one-sided objective.
We also make the following standard boundedness assumptions.
\begin{assumption} \label{asm:lipschitz}
    (\emph{Boundedness}) Given $G,\lambda, D > 0$, the movement coefficients satisfy $\lambda_t\in[0,\lambda]$, the domain 
    $\cW$ satisfies $\sup_{x,y\in\cW}\norm{x-y}\le D$, and the convex loss functions $f_1, \dots, f_T:\ww\to\R$ are $G$-Lipschitz, i.e., each $f_t$ satisfies $|f_t(x)-f_t(y)| \leq G\norm{x-y}$ for all $x,y\in\ww$.
\end{assumption}
Throughout the paper, we assume for simplicity that the bound on the movement coefficients $\lambda\ge \lambda_t$
is known \emph{a priori}, 
though in practice this 
factor can be estimated up to a constant factor 
using a standard doubling argument.

%% file: sections/movement_cost.tex
\section{\SecStronglyAdaptive}%
\label{sec:samc}

In this section, we introduce a general algorithm for obtaining strongly adaptive regret guarantees with time-varying movement costs.
At a high level, our approach uses a standard meta-learner/base-learner construction which combines interval-specific base learners using a meta-algorithm that aggregates their predictions. 
Unlike standard strongly adaptive constructions, however, the movement cost of the resulting predictor depends not only on the movements of the individual base learners, but also on changes in their aggregation weights. 
Our first step in \Cref{lemma:mc-redux,lemma:sa-mc-redux} is therefore to establish a regret decomposition that separates these two contributions, reducing the problem to a variant of sleeping experts in which the meta-algorithm itself incurs time-varying movement costs. 
Then, in \Cref{sec:sleeping}, we develop an algorithm for this meta-level problem, 
and in \Cref{sec:main} we combine it with 
interval-specific base learners and 
an adaptive blocking reduction to obtain our 
main strongly-adaptive regret guarantee.

The following lemma shows that the regret of the mixture predictor decomposes
into a base regret term and a meta-regret term. 
\begin{restatable}{lemma}{MCRedux}\label{lemma:mc-redux}
  Let \Cref{asm:lipschitz} hold.
  Suppose on each round $t \in [T]$
  we play $\wt = \sum_{i\in[N]}\pti\wt^{(i)}$, where $\wt^{(i)}\in\ww$ is the output of
  expert $\cA_{i}$ on round $t$ and $\pt\in\Delta_{N-1}$ is the output of a meta-algorithm $\cA_\meta$. Then for any
  $i\in[N]$ and any sequence $\cmp_{1},\ldots,\cmp_{T} \in \ww$,
  \begin{align*}
  R_{T}(\cmp_{1:T},\lambda_{1:T})
    \le
    \underbrace{\sumtT\inner{\gt, \wt^{(i)}\! - \cmp_{t}}+\lambda_t\norm{\wt^{(i)}\! - \wtmm^{(i)}}}_{\eqqcolon R_{T}^{\cA_{i}}(\cmp_{1:T}, \lambda_{1:T})}
      +\underbrace{\sumtT \inner{\ell_{t},\pt-e_{i}}+D\lambda_t\norm{\pt-\ptmm}_{1}}_{\eqqcolon R_{T}^{\meta}(e_{i}, \lambda_{1:T}^\meta)} ,
  \end{align*}
  where $g_t \in \partial f_t(w_t)$, $\lambda_t^\meta=D\lambda_t$, 
  and the surrogate loss vector $\ell_t \in \mathbb{R}^N$ is defined coordinate-wise by
$
\ell_{t,j}
\coloneqq
\inner{g_t, w_t^{(j)}}
+
\lambda_t \norm{w_t^{(j)} - w_{\tmm}^{(j)}}, \forall j \in [N].
$
\end{restatable}
The full proof is deferred to \Cref{app:mc-redux}. 
Unlike the decomposition of \citet{qiu2026parameterfree}, which aggregates expert predictions by summing their outputs and therefore requires an unconstrained domain and comparator-adaptive base algorithms, 
our decomposition applies to arbitrary convex feasible sets and allows for general base algorithms. 
In particular, the convex aggregation preserves feasibility and enables guarantees analogous to \citet{qiu2026parameterfree} in constrained settings, which their approach does not cover.

To obtain strongly-adaptive guarantees, we further extend \Cref{lemma:mc-redux}
to the sleeping-experts setting with movement costs. 
We associate each expert $\cA_i$ with a prescribed time interval, during which it is active and contributes to the mixture predictor. Accordingly, the meta-algorithm's regret against each expert is measured only over that expert's active rounds. 
We encode the experts' activity through an activation matrix $\alpha\in\Set{0,1}^{T\times N}$, where $\alpha_{t,i}=1$ if expert $\cA_i$ is active on round $t$
and $\alpha_{t,i}=0$ otherwise.
For any comparator distribution
$\cmp\in\Delta_{N-1}$ over experts, we define the sleeping-experts regret with
movement costs, with respect to the activation matrix $\alpha$, as
\begin{equation}
    R_T^\alpha(\cmp, \lambda_{1:T})
    \triangleq
    \sumtT \sum_{i\in[N]}
    \alpha_{t,i}\cmp_i
    \bigl(
        \inner{\ell_t,\pt}
        - \ell_{t,i}
        + \lambda_t\norm{\pt-\ptmm}_1
    \bigr) \;.\label{eq:mc-sleeping-regret}
\end{equation}
Note that this naturally generalizes the usual notion of sleeping experts regret, recovering the standard definition when $\lambda_t = 0$ for all $t$.

The next lemma shows that the meta-base decomposition yields an
interval guarantee whenever the competing expert is active on that interval.
\begin{restatable}{lemma}{SAMCRedux}\label{lemma:sa-mc-redux}
  Under the same assumptions as \Cref{lemma:mc-redux},
  let $[a,b]\subseteq[T]$ and let 
  the active rounds of $\cA_i$ coincide with interval $[a,b]$, i.e., $\alpha_{t,i}=\Indicator\Set{t\in[a,b]}$ for all $t \in [T]$. 
  Then for any sequence $\cmp_{a:b}=(\cmp_{a},\ldots,\cmp_{b})$ over $\calW$,
  \(
    R_{\ab}(\cmp_{a:b}, \lambda_{a:b})
    \le
       R_{\ab}^{\cA_{i}}(\cmp_{a:b},\lambda_{a:b}) +R_{T}^{\alpha}(e_{i},\lambda_{1:T}^{\meta}) \;.
  \)
\end{restatable}
The proof is deferred to \Cref{app:sa-mc-redux}. This lemma reduces the problem of strong adaptivity to two tasks: designing a sleeping-experts meta-algorithm and developing robust base learners for intervals, both of which incorporate movement costs. We address the meta-algorithm design first in \pref{sec:sleeping}.

%% file: sections/sleeping.tex
\subsection{\SecSleeping}%
\label{sec:sleeping}

Using the specialists reduction from \Cref{lemma:sa-mc-redux}, we know we can obtain strongly-adaptive guarantees
for movement-cost regret if we can control the 
movement-cost sleeping-experts regret $R_T^\alpha(\cmp,\lambda_{1:T})$ defined in \Cref{eq:mc-sleeping-regret}.
In particular, observe that if we can obtain an adaptive bound of the form
\begin{align}
    R_T^\alpha(\cmp,\lambda_{1:T})
    &=\tilde\cO\bigg(\sum_{i\in[N]}\cmp_i\sqrt{\sum_{t\in[T]} \alpha_{\ti}(\inner{\ell_t,e_i-\pt}^2+\lambda_t^2)}\bigg) \;,
    \label{eq:desired-bound}
\end{align}
then given a specialist $\cA_I$ defined 
precisely on $I=\ab$, choosing 
$\cmp=e_i$ automatically restricts 
the meta-regret bound to rounds
within $I$. Thus, under 
bounded surrogate losses and movement-cost coefficients, this yields
a meta-regret guarantee of order 
$\tilde\cO(\sqrt{\abs{I}})$, up to 
problem-dependent factors. Combined 
with \Cref{lemma:sa-mc-redux}, this 
reduces the problem of obtaining strongly adaptive regret to controlling the regret of the interval-specific base-learner.

To establish the desired sleeping-experts guarantee, we employ a reduction to online linear optimization (OLO) in the 
spirit of \citet{jun2017improved}. However, unlike their betting-based argument, our analysis is based on 
follow-the-regularized-leader (FTRL), 
inspired by the approach of \citet{jacobsen2022parameter}. 
This is a key distinction, as it is not obvious how to extend the betting-based approach to incorporate adaptive movement costs.

\begin{algorithm}[t!]
  \textbf{Input:}~$L>0$, value $\epsilon_i>0,~\forall i\in[N]$, activation map $t\mapsto \alpha_{t,i}$ supported on $[a_i,b_i]$ for each $i\in[N]$\\
  \textbf{Initialize:}~ $z_{1}=\zeros\in\R^{N}$\\
  \textbf{Define:}~
  \(\psi(z; V,\beta)=3\int_{0}^{z}\min_{0<\eta\le 1/L}\sbrac{\frac{\Log{x /\beta+1}}{\eta}+\eta V} \diff x\)\\
  \For{$t=1,\dots,T$}{
    Play
    $\pt = \frac{\zt\odot\alpha_t}{\norm{\zt\odot\alpha_t}_1}$ if $\norm{\zt\odot\alpha_t}_{1}\ne 0$ and $p_t=\tfrac{\alpha_t}{\norm{\alpha_t}_1}$ otherwise\\
    Observe loss vector $\ell_{t}\in\R^{N}$, movement coefficient $\lambda_\tpp\ge 0$\\
    \For{$i=1,\ldots,N$}{
        Define $\ellti^\lambda = \alpha_{t,i}(\inner{\ell_t,e_i-\pt}-\lambda_t\norm{\pt-\ptmm}_1)+\alpha_\ti(1-\alpha_\tppi)2\lambda_\tpp,~\forall i\in[N]$\\
        Set $\lambda_\tpp^{(i)}=2\alpha_\ti\alpha_\tppi\lambda_\tpp,\quad V_\tppi^\lambda = 4L^2 + \sum_{\tau=1}^t(\abs{\ell_{\tau,i}^\lambda}+\lambda_{\tau+1}^{(i)})^2$,\\
        \hphantom{Set }and $\beta_{\tppi}=\frac{\epsilon_i L}{\sqrt{V_\tppi^\lambda}\log^2(V_\tppi^\lambda/L^2)}$, $\forall i\in[N]$\\
        Update \(\ztppi = \argmin_{\z\in [0,1]}\sbrac{\brac{\sum_{s=1}^{t}\ell_\si^\lambda}\z_{i}+\psi(\z_{i}; V_{\tpp,i}^\lambda,\beta_\tppi)}\)
    }
  }
  \caption{Adaptive Sleeping Experts with Movement Costs}
  \label{alg:sleeping-mvmt-costs}
\end{algorithm}

\Cref{alg:sleeping-mvmt-costs} implements this construction for sleeping experts with time-varying movement costs, and the following theorem establishes its second-order regret guarantee.

\begin{restatable}{theorem}{SleepingMvmtCosts}\label{thm:sleeping-mvmt-costs}
  Suppose that $\norm{\ell_{t}}_{\infty}\le G_{\infty}$ and $\lambda_{t}\le \lambda$
  for all $t$.
  Then
  for any $\cmp\in \Delta_{N-1}$,
  \Cref{alg:sleeping-mvmt-costs} with
  $L=2G_{\infty}+5\lambda$ guarantees
  \begin{align*}
     R_T^{\alpha}(\cmp, \lambda_{1:T}) 
     &\le 
     \tilde\cO\Bigg(
     \Gamma(\epsilon)+
     \sum_{i\in[N]}u_i\sqrt{\sum_{t\in[T]} \alpha_\ti\Big((\ellti-\inner{\ell_t,\pt})^2+\lambda_t^2\Big)\Log{\tfrac{u_i \sqrt{V_{T+1,i}^\lambda}}{\epsilon_i L}+1}}
     \Bigg),
  \end{align*}
  where
  $\Gamma(\epsilon)= L\big(1+\sum_{i}\epsilon_{i}\big)$.
\end{restatable}

The proof is deferred to \Cref{app:sleeping-mvmt-costs}. \Cref{thm:sleeping-mvmt-costs} shows the second-order guarantee required in \Cref{eq:desired-bound}, with the regret against any expert depending only on the rounds where that expert is active. Our result strengthens the sleeping-experts guarantee
of \citet{jun2017improved} in two ways: it accommodates time-varying movement costs
and replaces their dependence on the number of active rounds with a second-order dependence on the 
realized instantaneous regrets. 
Indeed, as we show in \Cref{app:relationtojun}, their guarantee can be recovered by \Cref{thm:sleeping-mvmt-costs} via Cauchy-Schwarz inequality.

%% file: sections/main-result.tex
\section{\SecMain}
\label{sec:main}

In this section, we present the main result of the paper: a strongly-adaptive algorithm for movement-cost dynamic regret achieving $\tilde\cO\Brb{\sqrt{(D^2+DP_\ab)\sumab(\norm{\gt}^2+\lambda_t\norm{\gt})}}$ regret for every time interval $\ab\subseteq[T]$ simultaneously.

We first introduce one additional reduction.
A direct application of \Cref{thm:sleeping-mvmt-costs} in the context of the specialists reduction in \Cref{lemma:sa-mc-redux} naturally leads to bounds depending on
$\sum_t(\|g_t\|^2+\lambda_t^2)$. For time-varying movement costs, the sharper
quantity is instead
\(
    \sum_t (\|g_t\|^2+\lambda_t\|g_t\|),
\)
which can always be bounded in terms of the former quantity via AM-GM inequality, but can be significantly tighter 
when $\norm{\gt}\ll\lambda_t$ on average. 
In particular, this refinement prevents large movement penalties from dominating the regret bound when the observed gradients are small.
Moreover, as shown by 
\citet{qiu2026parameterfree}, obtaining the first-order dependence on a time-varying sequence of $\lambda_t$
is crucial for obtaining optimal guarantees in the delayed feedback setting, which we later discuss in \Cref{sec:delays}.

To obtain the desired first-order dependence on $\lambda_t$, we use an adaptive first-order blocking strategy inspired by \citet{qiu2026parameterfree}, presented in \Cref{alg: dyn_para_meta_2}. The
algorithm groups consecutive rounds into epochs and passes the aggregate
gradient of each completed epoch to the base learner. 
Crucially, our blocking condition depends on the sum of gradient norms $\sum_t\norm{\gt}$ rather than the 
norm of their sum $\norm{\sum_t\gt}$, as
in \citet{qiu2026parameterfree}. The former is monotone within each epoch, whereas the latter may decrease. This distinction is important for strong adaptivity, since an arbitrary interval $\ab$ may begin or end in the middle of an epoch---the monotonicity of our blocking condition allows us to control the contribution of these partial epochs. The following lemma formalizes this property, showing that regret on an arbitrary interval can be reduced to regret over the complete epochs contained within it, up to a lower-order boundary term. The proof is deferred to \Cref{app:firstorderreduction}.

\begin{algorithm}[t!]
  \textbf{Input:} 
   values $L>0$, 
  Geometric covering intervals $\cI$, algorithm $\cA_I$ 
  for each $I\in\cI$, meta-algorithm\\
  \hphantom{Input: }$\cA_\meta$ defined on $\Delta_{\cI}$ with parameters $\epsilon_I$ for each $I\in\cI$\\
  \textbf{Initialize:}
  $\tau = 1$, $\tlambda_\tau=0$, $\twtau=\zeros$, $\tg_\tau=\zeros$, $H_\tau=\mathbf{0}$, $\talpha_{\tau}=(\alpha_{\tau,I})_{I\in\cI}$, $s_\tau=1$,\\
  \hphantom{Initialize: }
  Initialize $\cA_\ab$ and initialize $\cA_\meta$ 
  with $\epsilon_{[a,b]} = \tfrac{1}{a^2 (1+\lfloor\log_2a\rfloor)}$ for each $[a,b]\in \cI$\\
    \For{$t = 1, 2, \dots$}{
        Play $\wt = \twtau$, observe $\gt\in\partial\ft(\wt)$ and $\lambda_{\tpp}$\\
        Update $\tgtau \leftarrow \tgtau+ g_t $ and $H_\tau \leftarrow H_{\tau}+\norm{\gt}$\\
        \If{$ H_{\tau} > \lambda_{t+1}$}{

        Set $\widetilde\lambda_{\tau+1} = \lambda_{t+1},\quad\talpha_{\taupp}=(\alpha_{\tppI})_{I\in\cI},\quad s_{\tau+1}=t+1$\\
        Send $(\tgtau,\tlambda_{\tau+1})$
        to $\cA_I$ and retrieve $\wtpp^{(I)}$ for each $I\in\Active(t)\cap \Active(s_\tau)$\\
        Send 
        $(\telltau, D\tlambda_{\taupp})$ to $\cA_\meta$ where $\telltauI = \talpha_{\tauI}\brac{\inner{\tgtau,\twtau^{(I)}}+\tlambda_{\tau}\norm{\twtau^{(I)}-\twtaumm^{(I)}}}$\\
        Update
        $\tilde w_{\tau+1}=\sum_{I\in\cI} p_{\tau+1,I}w_{\taupp}^{(I)},\quad
          H_{\tau+1}=0,\quad\tg_\taupp = \zeros,\quad\text{and}\quad\tau \gets \tau + 1$
        }
    }
    \caption{Strongly-Adaptive Online Learning with Time-varying Movement Costs}
  \label{alg:ol-mvmt-costs}
\end{algorithm}

\begin{minipage}{\columnwidth}
\begin{restatable}{lemma}{SABlocking}\label{lemma:sa-blocking}
    Let \Cref{asm:lipschitz} hold. 
    Let $(I_\tau)_{\tau}$ be the set of intervals defined during the execution of the blocking reduction (\Cref{alg: dyn_para_meta_2}), and let 
    $\cE_\ab=\Set{\tau: I_\tau\subseteq\ab}$ denote the set of 
    complete epochs contained within $\ab$.
    For any comparator sequence $(u_t)_{t\in\ab}$ in $\calW$, \Cref{alg: dyn_para_meta_2} with any base algorithm $\calA$ guarantees that
    \begin{align*}
        R_\ab(\cmp_{a:b},\lambda_{a:b})\le 3(D+P_\ab)(G+\lambda_{\max}^\ab)+\sum_{\tau\in\cE_\ab}\Brb{\inner{\tgtau,\twtau-\tcmp_\tau} + \tlambda_\tau\norm{\twtau-\tw_\taumm}},
    \end{align*}
    where $\lambda_{\max}^\ab=\max_{t\in\ab}\lambda_t$, and for each complete epoch 
    $\tau\in\cE_\ab$, the epoch-level comparator $\widetilde{\cmp}_\tau\in \Set{u_t:t\in I_\tau}$ may be chosen arbitrarily. 
    All other quantities are defined in \Cref{alg: dyn_para_meta_2}. 
\end{restatable}
\end{minipage}

Now applying this reduction on top of the 
specialists reduction in \Cref{sec:samc} with the sleeping-experts algorithm of \Cref{sec:sleeping}, we are ready to derive the main result.
We proceed by constructing 
a set of experts $\Set{\cA_{I} : I\in\cI}$ 
such that $\cA_{I}$ is active 
only on interval $I\in \cI$, where $\cI$ is the standard
set of geometric covering intervals \citep{daniely2015strongly}:
\(
    \cI = \bigcup_{k\in\Set{0,1,\ldots}}\Set{[i 2^k,(i+1)2^k-1]:i=1,2,\ldots}.
\)
The set of active intervals at time $t$ is then
$\Active(t)=\Set{I\in\cI : t\in I}$, and we will set $\alpha_{t,I}=\Indicator\Set{I\in \Active(t)}$ as the activation indicator. 
For our purposes here, it is sufficient to take each $\cA_I$ to be an instance of adaptive (sub)-gradient descent  applied on interval $I$ (see \Cref{app:adaptive-gd}). 
Note that the experts are updated only 
on complete epochs;
 this avoids contaminating the feedback of an expert $\cA_\ab$ with feedback from
 outside $\ab$.
Composing these results leads to the design of \Cref{alg:ol-mvmt-costs} with the following strongly-adaptive guarantee, 
proven in \Cref{app:main}.

\begin{restatable}{theorem}{Main}\label{thm:main}
  Let \Cref{asm:lipschitz} hold and let $L=D(2G+9\lambda)$. 
  \Cref{alg:ol-mvmt-costs} applied with
  base algorithms $\cA_I$ implementing \Cref{alg:adaptive-gd} with $\alpha=D$ and $\cA_\meta$ implementing \Cref{alg:sleeping-mvmt-costs}
  guarantees, for any $\ab\subseteq[T]$,  
  any $\cmp_a,\ldots,\cmp_b \in \ww$, and any $\lambda_{a},\ldots,\lambda_b \in [0,\lambda]$, that
   \begin{align*}
       &R_\ab(\cmp_{a:b},\lambda_{a:b})
       \le
        \hat \cO\bigg(
        (D+P_\ab)\log^2(b+1)(G+\lambda)\\
        &\qquad
        + \sqrt{D^2+DP_\ab}\min\bigg\{\sqrt{(G^2+G\lambda)\abs{b-a}\log(b+1)},
        \sqrt{V_\ab^\lambda}\log(b+1)\bigg\}
        \bigg),
   \end{align*}
   where $V_\ab^\lambda=\sum_{t\in \ab}(\norm{\gt}^2+\lambda_t\norm{\gt})$ and $P_\ab=\sum_{t=a+1}^b\norm{\cmp_t-\cmp_\tmm}$.
\end{restatable}
\Cref{thm:main} provides two complementary guarantees: a worst-case bound depending on the interval length $\abs{b-a}$,
and a data-dependent bound depending on the realized 
quantity $V_\ab^\lambda=\sumab(\norm{\gt}^2+\lambda_t\norm{\gt})$.
The worst-case bound recovers the best-known strongly-adaptive static and dynamic guarantees
of order $\cO\brb{G\sqrt{(D^2+DP_\ab)\abs{b-a}\log(b)}}$ in the standard OCO setting as a special case, by setting $\lambda_t=0$ for all $t$ \citep{jun2017improved,zhang2021dual}. Moreover, when movement penalties are present, our bound improves upon the bound of \citet{zhang2022adversarial}, which incurs a sub-optimal penalty of $\log(b)$ rather than the best-known dependency of $\sqrt{\log(b)}$.

On the other hand, our adaptive bound $\cO\Brb{\sqrt{(D^2+DP_\ab)\sumab\norm{\gt}^2+\lambda_t\norm{\gt}}\log(b)}$ also extends two distinct lines of prior work.  While 
\citet{qiu2026parameterfree} establish data-dependent dynamic regret guarantees over the whole horizon, our
result extends these guarantees to hold simultaneously over every sub-interval, 
at the cost of an additional $\sqrt{\log(b)}$ penalty over the non-adaptive guarantee. Moreover, 
even for fixed comparators, our 
guarantee obtains a sharper dependence on the observed gradients and movement coefficients than the strongly adaptive bound of \citet{zhang2022adversarial}. Thus, our result combines the interval-wise 
adaptivity of the latter with the data-dependent guarantees of the former, while accommodating both 
time-varying movement-cost coefficients and dynamic comparator sequences.

%% file: sections/applications.tex
\section{\SecApplications}\label{sec:applications}

In this section, we derive strongly adaptive regret guarantees for two further problem settings: OCO with delayed feedback and OCO with time-varying memory.
These results are obtained via extended reductions to OCO with time-varying movement costs, and generalize the reductions of \cite{qiu2026parameterfree} to hold on arbitrary sub-intervals.

\subsection{Strongly Adaptive Regret with Delayed Feedback}%
\label{sec:delays}

We first study strongly adaptive regret for OCO with delayed feedback.
The online protocol proceeds as follows.
On each round $t \in [T]$, the learner chooses a decision $w_t \in \mathcal{W}$, while the environment simultaneously selects a convex loss function $f_t:\mathcal{W}\to\mathbb{R}$.
The learner then suffers the loss $f_t(w_t)$. In contrast to the standard OCO setting, however, the gradient
$
    g_t \in \partial f_t(w_t)
$
is not necessarily observed immediately, and is instead revealed only at the end of round $t+d_t$, where $d_t \ge 0$ denotes an arbitrary delay that is unknown to the learner. Without loss of generality, we assume that $t+d_t \le T$ for every $t \in [T]$, so that all feedback is received by the end of the time horizon. The learner’s performance on an interval $\ab \subseteq [T]$ is measured by the strongly adaptive regret
\(
    \RD_\ab(u_{a:b})
    =
    \sum_{t\in\ab} \bigl(f_t(w_t)-f_t(u_t)\bigr)
\),
where $u_{a:b}=(u_a,\ldots,u_b)$ is any comparator sequence in this interval. We use the superscript ``$\del$'' to emphasize that the regret is evaluated under delayed feedback. 
For each round \(t\), let \(o_t \defeq \{\tau \in \mathbb{Z}_+ : \tau+d_\tau < t\} \subseteq [t-1]\) denote the set of past rounds whose gradients have already been observed before round \(t\).
We further define
\(m_t \defeq [t-1]\setminus o_t\) as the set of past rounds whose gradients are still outstanding before round \(t\), and 
define $\sigma=\max_t\abs{m_t}$.

With this setup, the delayed-feedback protocol can be equivalently described by saying that, at the end of round \(t\), the learner receives the batch of newly available gradients \(\{g_\tau\}_{\tau\in o_{t+1}\setminus o_t}\).
Let $h_t = \sum_{\tau\in o_\tpp\setminus o_t}\gtau$ denote their sum. The
following lemma shows that the delayed-feedback regret on any interval can be bounded in terms of an undelayed movement-cost objective with gradients $h_t$
and movement-cost coefficients $\lambda_t=G\abs{m_t}$, up to 
an additional 
lower-order penalty.

\begin{restatable}{lemma}{DelayedRedux}\label{lemma:delayed-redux}
  Let \Cref{asm:lipschitz} hold and let $[a,b] \subseteq [T]$ be any interval.
  Then
  \begin{align*}
    \RD_{[a,b]}(u_{a:b})
    &\le
     G(2D+P_{[a,b]})\sigma+
     \sum_{t=a}^b \inner{h_t, \wt-\cmp_{t}} + \sum_{t=a+1}^b \lambda_t\norm{\wt-\wtmm},
  \end{align*}
  where $h_{t}=\sum_{\tau\in o_{\tpp}\setminus o_{t}}\gtau$ and $\lambda_{t}=G\abs{m_{t}}$. 
\end{restatable}
The result generalizes the analogous reduction from \citet{qiu2026parameterfree} to hold on an arbitrary sub-interval $\ab$, and requires 
additional care to handle boundary conditions unique to the strongly-adaptive setting. A direct application of \Cref{thm:main} then leads to the following guarantee for the delayed feedback setting.

\begin{minipage}{\columnwidth}
\begin{restatable}{theorem}{MainAnyABDelay}\label{thm:main-delay}
    Let \Cref{asm:lipschitz} hold, and let $H \ge G(\sigma+1)$ and $\lambda=G\sigma$.
    Initializing \Cref{alg:ol-mvmt-costs} with $L=D(2H+9\lambda)$, and feeding it gradients $h_t = \sum_{\tau \in o_{t+1} \setminus o_t} g_\tau$ and movement coefficients $\lambda_t = G|m_t|$,
    guarantees that for any $\ab\subseteq[T]$ and any $\cmp_a,\ldots,\cmp_b \in \ww$ that
    \begin{align*}
        \RD_{\ab}(u_{a:b})
        &\le \tilde\cO\biggl(
            (D + P_{\ab})  L_{\sigma,a}
            + G\sqrt{(D^2+DP_\ab)\brb{|b-a| + \dtot^{\ab}}} 
        \biggr) \;,
    \end{align*}
    where $L_{\sigma,a} = H \lor G \brb{\sigma + \dmax^{[1,a]}}$ and $P_\ab=\sum_{t=a+1}^b\norm{\cmp_t-\cmp_\tmm}$,
    and 
    $d_{\mathrm{tot}}^{[a,b]}=\sum_{t=a}^b \min\Set{d_t, b-t}$ denotes the interval-restricted total delay.
\end{restatable}
\end{minipage}

The proofs of \Cref{lemma:delayed-redux} and \Cref{thm:main-delay} are deferred to \Cref{apx:delays}. The leading term $\cO\Big(G\sqrt{(D^2+DP_\ab)(|b-a|+\dtot^\ab)}\Big)$ achieves the minimax-optimal dependence on interval-restricted cumulative delay and recovers the standard strongly-adaptive dynamic regret bound when delays vanish. While \citet{wan2025black} previously established strongly-adaptive regret under arbitrary delays, their guarantees apply to static comparators and depend on the global maximum delay. In contrast, our bound accommodates time-varying comparators on every interval, with an optimal dependence of $\sqrt{\dtot^\ab}\le \sqrt{\max_t d_t\abs{b-a}}$ in its leading term.

Our bound also incurs a lower-order term depending on $d_{\max}^{[1,a]} = \max_{t \in [1,a]} d_t$, accounting for feedback originating before the interval. Whether this dependence is intrinsic or can be removed through a more direct analysis remains open.

\begin{remark}
    The result in \Cref{thm:main-delay} requires prior knowledge of $\lambda\ge G\sigma=G\max_t\abs{m_t}$, 
    which may not be known in practice. However, this parameter can be estimated 
    up to a constant factor 
    using the standard doubling argument discussed in \Cref{sec:setting}, applied after the delays-to-movement-costs reduction (\Cref{lemma:delayed-redux}).
\end{remark}

\subsection{Strongly Adaptive Regret with Time-Varying Memory}%
\label{sec:memory}
In this section, we establish strongly adaptive regret guarantees for OCO with time-varying memory.
The interaction follows the standard $T$-round protocol.
In each round $t$, the learner chooses a decision $w_t \in \mathcal{W}$.
At the same time, the environment selects a loss function
$
f_t:\mathcal{W}^{\msize_t+1}\to\mathbb{R},
$
which depends on the current decision together with the preceding $\msize_t$ decisions, where
$\msize_t\in\{0,\ldots,t-1\}$ denotes the memory length at round $t$.
We assume that the full sequence of memory lengths $\{\msize_t\}_{t=1}^T$ is known to the learner in advance, and define the maximum memory length by $B\triangleq\max_t b_t$.
The learner then suffers the loss $f_t(w_{t-\msize_t:t})$, where
$
w_{t-\msize_t:t}\defeq (w_{t-\msize_t},\ldots,w_t),
$
and subsequently observes the entire loss function $f_t$.
The learner aims to minimize strongly adaptive regret against an arbitrary comparator sequence over any interval.
This regret is defined as
\(
    \RM_\ab(u_{1:T})
    \defeq
    \sumab
    \brb{
        f_t(w_{t-\msize_t:t}) -
        f_t(u_{t-\msize_t:t})
    }
\)
for any interval $\ab \subseteq [T]$; note that it takes the entire comparator sequence $\cmp_{1:T}$ as argument since the loss at some round might depend on points relative to rounds before the start of the interval.

For an interval $\ab$, we also define the clipped memory length at round $t$ by
$
    \msize_t^\ab
    \defeq
    \min\{\msize_t,t-a\}.
$
Thus, $\msize_t^\ab$ counts only the portion of the memory window that lies within the interval $\ab$.
We also define the maximum original memory length over $\ab$ as
$
    B_\ab
    \defeq
    \max_{t\in[a,b]} \msize_t,
$
and the maximum clipped memory length on $[a,b]$ as
$
    \bmax^{[a,b]}
    \defeq
    \max_{t\in[a,b]} \msize_t^{[a,b]}.
$
By definition, $\bmax^{[a,b]}\le B_{[a,b]}$. 

Following \citet{anava2015online,zhao2023non}, 
 we place the following structural assumptions on the loss functions: we assume that the losses are coordinate-wise Lipschitz, and 
 that the induced 
 unary losses $\hat \ft(w)\triangleq\ft(w,\ldots,w)$
 are convex and $H$-Lipschitz.
\begin{assumption}[Coordinate-wise Lipschitzness]
\label{asm:coordwise-lipschitz}
For every $t\in[T]$, the loss function $f_t$ is coordinate-wise $G$-Lipschitz for some $G\ge 0$.
That is, for any
$
x=(x_0,x_1,\ldots,x_{\msize_t})
$
and
$
y=(y_0,y_1,\ldots,y_{\msize_t})
$
in $\calW^{\msize_t+1}$,
$
    \left|
        f_t(x_0,\ldots,x_{\msize_t})
        -
        f_t(y_0,\ldots,y_{\msize_t})
    \right|
    \le
    G \sum_{i=0}^{\msize_t} \left\|x_i-y_i\right\|
$~.
\end{assumption}

\begin{assumption}[Properties of unary losses]
\label{asm:unary-properties}
The induced unary loss $\widehat f_t$ is convex,
and the gradients of the unary losses are uniformly bounded: there exists $H > 0$ such that $
    \sup_{w\in\ww} \norm{\nabla \widehat f_t(w)} \le H$ for all $t\in[T]$.
\end{assumption}

Note that by \Cref{asm:coordwise-lipschitz}, we may assume without loss of generality that $H\le G(B+1)$.
The following lemma reduces regret on any interval to an instance of OCO with time-varying movement costs up to additive lower-order terms.
\begin{restatable}{lemma}{MemoryReduction}\label{lemma:memory-redux}
    Let $[a,b] \subseteq [T]$ be any interval.
    Then
    \[
        \RM_{[a,b]}(u_{1:T}) \le GP_{[a,b]}(\bmax^{[a,b]})^2 + 2GDB_{[a,b]}^2 + \sum_{t=a}^b \inner{h_t, \wt-\cmp_t} + \sum_{t=a+1}^b \lambda_t \norm{\wt-\wtmm} \;,
    \]
    where $h_t = \nabla\hat f_t(w_t)$, $\xi_t = \sum_{s=t}^T \mathbb{I}\bcb{t \ge s-\msize_s+1}\brb{t-\brb{s-\msize_s}}$, and $\lambda_t = G\xi_t$.
\end{restatable}

The lemma plays a role analogous to that of \Cref{lemma:delayed-redux} from \Cref{sec:delays}, reducing the OCO with memory 
problem to a certain OCO with movement costs problem.
Applying \Cref{thm:main} in this movement costs problem leads to the following guarantee.

\begin{minipage}{\columnwidth}
\begin{restatable}{theorem}{MainAnyABMemory}\label{thm:main-memory}
    Initializing \Cref{alg:ol-mvmt-costs} from \Cref{thm:main} with Lipschitz constant $H$ and $\lambda = GB^2$, and feeding it gradients $h_t = \nabla \hat f_t(w_t)$ and movement coefficients $\lambda_t = G\xi_t$, where $\xi_t$ is defined as in \Cref{lemma:memory-redux},
    guarantees for any $\ab\subseteq[T]$ and any $\cmp_1,\ldots,\cmp_T \in \ww$ that
    \begin{align*}
    \RM_{\ab}(u_{1:T})
    &\le \tilde\cO\Bigg(
        (D + P_{\ab})\big(H+GB^2\big)
        + \sqrt{(D^2+DP_\ab)\min\Set{V_\ab^{(1)},V_\ab^{(2)}}}
    \Bigg),
    \end{align*}
    where 
    $P_\ab=\sum_{t=a+1}^b\norm{\cmp_t-\cmp_\tmm}$, and we define
    $V_\ab^{(1)}=\cO\Big((H+GB^2)\sumab\norm{h_t}\Big)$, and $V_\ab^{(2)}=\cO\Big(\sum_{t\in\ab}\norm{h_t}^2 + HG\min\Set{B^2\abs{b-a},B^4+\sum_{t\in\ab}\brb{\msize_t^\ab}^2}\Big)$.
    \end{restatable}
\end{minipage}

The proofs of both \Cref{lemma:memory-redux} and \Cref{thm:main-memory} are deferred to \Cref{apx:memory}.
The two branches of the minimum in \Cref{thm:main-memory} recover different guarantees from prior works. 
The first branch, based on $V^{(1)}_\ab$,
yields a data-dependent bound involving the unary gradients. Recalling $H\le G(B+1)$ by \Cref{asm:coordwise-lipschitz}, this branch gives 
\begin{align*}
    \RM_\ab(\cmp_{1:T})
    &=
    \tilde\cO\bigg(
    (D+P_\ab)G(1+B^{2.5})
    +B\sqrt{G(D^2+DP_\ab)\sumab\norm{h_t}}
    \bigg),
\end{align*}
which recovers the bound of \citet{zhang2022adversarial} when $P_\ab=0$, while improving the memory dependence from $B^3$ to $B^{2.5}$ in the lower-order term.

The second branch, based on $V^{(2)}_\ab$,
recovers, up to logarithmic factors and additive lower-order terms, the optimal dependence on the memory length and comparator path-length established by \citet{zhao2023non}. In particular, our regret bound scales as
\[
    \RM_\ab(\cmp_{1:T})
    =
\widetilde{O}\left(
B\sqrt{HG(D^2+DP_{[a,b]})|b-a|}
+
GB^2(D+P_{[a,b]})
\right),
\]
simultaneously over all intervals $\ab\subseteq[T]$.
Unlike \citet{zhao2023non}, our guarantees are 
strongly adaptive, accommodate arbitrary time-varying memory lengths, and admit refined data-dependent bounds involving the observed unary gradients and interval-clipped memory lengths.

%% file: sections/discussion.tex
\section{Discussion}%
\label{sec:discussion}

In this work, we presented new strongly-adaptive guarantees for dynamic regret with time-varying movement costs. Using this result, we derive strongly-adaptive dynamic regret guarantees under arbitrary delayed feedback, with a refined dependence on interval-restricted cumulative delays, as well as strongly-adaptive guarantees for OCO with time-varying memory.

Several interesting directions remain open. Our strongly adaptive OCO guarantee achieves the desired first-order dependence on the movement-cost coefficients through the quantity $\sumab\norm{\gt}^2+\lambda_t\norm{\gt}$. However, 
it is unclear whether our blocking argument can be extended to obtain the analogous dependence
$\sumab\inner{\ell_t,\pt-e_i}^2 + \lambda_t \abs{\inner{\ell_t,\pt-e_i}}$
directly in the sleeping-experts problem with movement costs. The main obstacle is that the sleeping-experts reduction does not readily accommodate separate coordinate-wise movement costs.

Another direction concerns the dependence on the quantities outside the interval of interest. In particular, our delayed-feedback guarantees contain a boundary term involving the maximum delay $d_{\max}^{[1,a]}$ before
the interval $\ab$. Moreover, our approach requires
prior knowledge of a global bound $\lambda\ge \max_t\lambda_t$, which translates into corresponding requirements on the maximum number of outstanding gradients $\sigma_{\max}=\max_t\abs{m_t}$ and the maximum memory length $B=\max_t b_t$
in the delayed-feedback and memory settings, respectively. An interesting direction for future work is to determine whether these global dependencies can be replaced by strictly interval-local quantities.

%% file: sections/more_related.tex
\section{Additional Related Works}\label{app:more_related}
In this section, we provide more detailed discussion of relevant related works.

\paragraph{Dynamic and adaptive regret.}
Dynamic regret has been extensively studied as a measure of performance in non-stationary OCO. The extension of regret to handling a \emph{sequence} of comparators dates back to \citet{herbster1998trackingexpert,herbster1998trackingregressor} in the regression and learning-from-expert-advice settings. 
\citet{zinkevich2003online} extended the scope to the more general OCO setting, and showed that online gradient descent can obtain 
$R_T(\cmp_{1:T})=\cO((D+P_T)\sqrt{T})$ dynamic regret, where $P_T=\sum_{t=2}^T\norm{\cmp_t-\cmp_\tmm}$ is the path-length of the comparator sequence. \citet{yang2016tracking} later showed that the bound can be improved to $\cO(\sqrt{(D^2+DP_T) T})$ given prior knowledge of $P_T$, and \citet{zhang2018adaptive}
were the first to obtain this bound without prior knowledge of $P_T$.

Another important line of work studies \textit{adaptive regret}. \citet{hazan2007adaptive} first introduced adaptive regret defined as the maximum regret over any contiguous interval, and later obtained an adaptive regret bound of $\mathcal{O}(\sqrt{T \log^3 T})$  for convex losses and $\mathcal{O}(\log^2 T)$ for exp-concave losses \citep{hazan2009efficient}. A key limitation of adaptive regret is that its guarantee is expressed in terms of the horizon $T$, rather than the length of the interval under consideration. Consequently, the bound can be uninformative on short intervals. To address this drawback, \citet{daniely2015strongly} introduced \emph{strongly adaptive regret}. This stronger criterion measures regret explicitly as a function of the interval length, and is therefore more suitable for evaluating performance on both short and long intervals. The algorithm proposed by \citet{daniely2015strongly} achieves an $\cO(\sqrt{\tau}\log T)$ regret for each interval of length $\tau\in[T]$ simultaneously. Subsequent work \citep{jun2017improved, zhang2020minimizing} improved this guarantee by introducing new meta-algorithms, reducing the strongly adaptive regret to $\cO(\sqrt{\tau \log T})$.

\paragraph{OCO with movement costs.} 
OCO with movement costs, also known as smoothed online convex optimization \citep{goel2017thinking,chen2018smoothed,goel2019online, li2020online}, has received considerable attention in recent years. In this framework, on round $t$ the learner incurs both the loss $\ell_t(w_t)$ as well as a movement cost $\lambda\norm{\wt-\wtmm}$ for deviating from its decision in the previous round. A closely related formulation is OCO with memory \citep{merhav2002sequential}, in which each loss may depend on a history of past decisions. Stability plays a central role in this setting, and many algorithms achieve it by reducing the problem to OCO with movement costs \citep{anava2015online}.

\paragraph{Online learning with delayed feedback.} 
Delayed feedback has been extensively studied in online learning. Early work by \citet{weinberger2002delayed} and \citet{joulani2013online} developed black-box reductions from OCO with delayed feedback to standard OCO, obtaining regret bounds of order $\cO(\sqrt{T+\dmax T})$, where $\dmax$ denotes the maximum delay. For convex losses, \citet{quanrud2015online} improved this dependence by adapting online gradient descent, achieving a regret bound of $\cO(\sqrt{T+\dtot})$, where $\dtot$ is the total cumulative delay, which can be significantly smaller than $\dmax T$. There is also a parallel line of work that focuses on obtaining adaptive regret guarantees and fast rates in delayed-feedback settings \citep{NIPS2014_5cce8ded, joulani2016delay,wan2022online,qiuexploiting}. 
Feedback delays also led to the development of specialized techniques for a broad range of problem settings related to multi-armed bandits \citep{NEURIPS2019_0e4f5cc9,pmlr-v108-zimmert20a,NEURIPS2022_4cd1d0d9,pmlr-v151-van-der-hoeven22a,pmlr-v202-esposito23a,pmlr-v195-hoeven23a,pmlr-v267-zhang25ca,ryabchenko2025capacity}.

%% file: sections/appendix-samc.tex
\section{Proofs for Section~\ref{sec:samc} \texorpdfstring{\scalebox{0.95}{(\SecStronglyAdaptive)}}{(\SecStronglyAdaptive)}}%
\label{app:samc}

\subsection{Proof of Lemma~\ref{lemma:mc-redux}}%
\label{app:mc-redux}

In this section we prove a general specialists reduction for the movement-cost regret. The reduction is over $[1,T]$, and we later specialize the result to sub-intervals $[a,b]$ in \Cref{lemma:sa-mc-redux}.

\begin{minipage}{\columnwidth}
\MCRedux*
\end{minipage}

\begin{proof}
Following the pre-interval convention in \Cref{sec:setting}, define $p_0=p_1$ and $w_0^{(j)}=w_1^{(j)}$ for all $j\in[N]$. Then $w_0=\sum_{j\in[N]}p_{0,j}w_0^{(j)}=w_1$, so the mixture identity also holds at $t=0$ and all movement terms at $t=1$ vanish.

By convexity of $f_t$ for all $t\in[T]$ (\Cref{asm:lipschitz}), we know that for any $i\in [N]$
\begin{align*}
    R_{T}(\cmp_{1:T},\lambda_{1:T})
    &=
    \sumtT\left(\ft(\wt)-\ft(\cmp_t)+\lambda_t\norm{\wt-\wtmm}\right)\\
    &\le
    \sumtT\left(\inner{\gt,\wt-\cmp_t}+\lambda_t\norm{\wt-\wtmm}\right)\\
    &=\sumtT\left(\inner{\gt,\sum_{j=1}^N p_{t,j}w_t^{(j)}-\cmp_t}+\lambda_t\norm{\wt-\wtmm}\right)\\
    &=
    \sumtT\inner{\gt,\wt^{(i)}-\cmp_t}+\sumtT\inner{\elltilde_t,\pt-e_i}+\sumtT\lambda_t\norm{\wt-\wtmm},
\end{align*}
where $\gt\in\partial\ft(\wt)$ and
define $\elltilde_t\in \R^N$ as
$
    \elltilde_{t,j} \triangleq \inner{\gt,\wt^{(j)}}
$ for $j \in [N]$ .
Moreover, observe that the movement costs can be bounded as
\begin{align*}
      \norm{\wt-\wtmm}
    &=
    \left\|\sum_{j\in[N]}\left( \ptj\wt^{(j)}-\ptmmj\wtmm^{(j)}\right)\right\|\\
    &=
    \left\|\sum_{j\in[N]} \ptj(\wt^{(j)}-\wtmm^{(j)})+\sum_{j\in[N]}(\ptj-\ptmmj)\wtmm^{(j)}\right\|\\
    &\le
    \sum_{j\in[N]}\ptj\norm{\wt^{(j)}-\wtmm^{(j)}} +\sum_{j\in[N]} \norm{\wtmm^{(j)}} \cdot \left|\ptj-\ptmmj\right| \tag{triangle inequality}
    \\
    &\le
      \sum_{j\in[N]}\ptj\norm{\wt^{(j)}-\wtmm^{(j)}} + D\cdot \|{\pt-\ptmm}\|_1 \;, 
\end{align*}
The last inequality uses $0\in\mathcal W$ and $\sup_{x,y\in\mathcal W}\norm{x-y}\le D$, which imply $\|w_{t-1}^{(j)}\|\le D$ for every $j\in[N]$. Summing over $t$,
\begin{align*}
   \sumtT \lambda_t \norm{\wt-\wtmm}
   &\le
   \sumtT \sum_{j\in[N]}\ptj\lambda_t\norm{\wt^{(j)}-\wtmm^{(j)}}+\sumtT D\lambda_t\norm{\pt-\ptmm}_1\;.
\end{align*}
Hence,
we are able to bound $R_{T}(\cmp_{1:T}, \lambda_{1:T})$ as follows:
\begin{align*}
    R_{T}(\cmp_{1:T}, \lambda_{1:T})
    &\le 
    \sumtT\inner{\gt,\wt^{(i)}-\cmp_t}+\sumtT\inner{\elltilde_t,\pt-e_i}\\
    &\qquad
    +\sum_{t=1}^T\sum_{j\in[N]}\lambda_t\ptj\norm{\wt^{(j)}-\wtmm^{(j)}} + \sumtT D\lambda_t\|{\pt-\ptmm}\|_1 \\
   &\overset{(*)}{=} 
    \sumtT\sbrac{\inner{\gt,\wt^{(i)}-\cmp_t}+\lambda_t\norm{\wt^{(i)}-\wtmm^{(i)}} }\\
    &\qquad
    +\sumtT\sbrac{\inner{\ell_t,\pt-e_i}+D\lambda_t\norm{\pt-\ptmm}_1  }
\end{align*}
where $(*)$ is by the definition of $\ell_{t,j}
\triangleq
\widetilde{\ell}_{t,j}
+
\lambda_t \left\| w_t^{(j)}-w_{\tmm}^{(j)} \right\|
$ for all $j \in [N]$.
\end{proof}

\subsection{Proof of Lemma~\ref{lemma:sa-mc-redux}}%
\label{app:sa-mc-redux}

In this section we prove our specialists reduction for an arbitrary sub-interval $\ab$. The result follows via a special case of \Cref{lemma:mc-redux}.
\SAMCRedux*
\begin{proof}
  Applying \Cref{lemma:mc-redux} over the interval $I=[a,b]$ where expert $\cA_i$ is active,
  we have
  \begin{equation*}
    R_{\ab}(\cmp_{a:b},\lambda_{a:b})
    \le
      R_{[a,b]}^{\cA_{i}}(\cmp_{a:b},\lambda_{a:b}) + R_{[a,b]}^{\meta}(e_{i},D\lambda_{a:b}) \;.
  \end{equation*}
  Note that because $\cA_{i}$ is active only on rounds $t\in[a,b]$,
  the meta regret is equivalent to
  \begin{align*}
    R_{[a,b]}^{\meta}(e_{i},\lambda_{a:b}^{\meta})
    &=
      \sum_{t\in[a,b]}\inner{\ell_{t},\pt-e_{i}}+D\lambda_{t}\norm{\pt-\ptmm}_1\\
    &=
      \sumtT\alpha_{ti}\bsb{\inner{\ell_{t},\pt-e_{i}}+D\lambda_{t}\norm{\pt-\ptmm}_{1}}\\
    &=
      \sumtT\sum_{j\in[N]}\sbrac{\alpha_{tj}e_{ij}\brac{\inner{\ell_{t},\pt}-\ell_{tj}}+D\lambda_{t}\alpha_{tj}e_{ij}\norm{\pt-\ptmm}_{1}}\\
    &=
      \sumtT\sum_{j\in[N]}\sbrac{\alpha_{tj}e_{ij}\brac{\inner{\ell_{t},\pt}-\ell_{tj}+D\lambda_{t}\norm{\pt-\ptmm}_{1}}}\\
    &=
      R_{T}^{\alpha}(e_{i},D\lambda_{1:T}).
      \end{align*}
      Plugging this back into 
      the previous display
      and recalling that 
      $\lambda_t^\meta=D\lambda_t$ proves the result. \qedhere

\end{proof}

%% file: sections/appendix-sleeping.tex
\section{Proofs for Section~\ref{sec:sleeping} (\SecSleeping)}%
\label{app:sleeping}
\subsection{Proof of Theorem~\ref{thm:sleeping-mvmt-costs}}%
\label{app:sleeping-mvmt-costs}

\begin{minipage}{\columnwidth}
\begin{manualtheorem}{\ref{thm:sleeping-mvmt-costs}}
  Suppose that $\norm{\ell_{t}}_{\infty}\le G_{\infty}$ and $\lambda_{t}\le \lambda$
  for all $t$.
  Then
  for any $\cmp\in \Delta_{N-1}$,
  \Cref{alg:sleeping-mvmt-costs} with
  $L=2G_{\infty}+5\lambda$ guarantees
  \begin{align}\label{eq:sleeping-explicit-bound}
     R_T^{\alpha}(\cmp, \lambda_{1:T}) 
      &\le
      \cO\Bigg(
        \Gamma(\epsilon)+
      \sum_{i\in [N]}
         u_i\sbrac{\sqrt{V_{T+1,i}^\lambda\Log{u_i\Lambda_{T+1,i}+1}}
        \maxOp
        L \Log{u_i\Lambda_{T+1,i}+1}}
      \Bigg),
  \end{align}
  where
  \begin{align*}
      V_{T+1,i}^\lambda
      &=
      \cO\brac{L^2+\sumtT \alpha_{\ti}\sbrac{\inner{\ell_t,e_i-\pt}^2+\lambda_t^2}},\\
    \Lambda_{T+1,i} &= \frac{\sqrt{V_{T+1,i}^{\lambda}}\log^{2}(V_{T+1,i}^{\lambda}/L^{2})}{\epsilon_{i} L}\\
    \Gamma(\epsilon)&= L\Big(1+\sum_{i\in[N]}\epsilon_{i}\Big)
  \end{align*}
\end{manualtheorem}
\end{minipage}

\begin{proof}

  Let $\zt^\alpha$ denote the vector with entries $\zti^\alpha=\alpha_{ti}\zti$.
  Since $\zt\in[0,1]^N$, the vector $\zt^\alpha$ is nonnegative. If $\norm{\zt^\alpha}_1>0$, \Cref{alg:sleeping-mvmt-costs} sets $\pt=\zt^\alpha/\norm{\zt^\alpha}_1$. Otherwise, it uses $\pt=\alpha_t/\norm{\alpha_t}_1$. Provided that at least one expert is active on each round, both cases give nonnegative entries summing to one, and hence $\pt\in\Delta_{N-1}$.

  First, we express the sleeping-expert regret of the choices $\pt$ made by \Cref{alg:sleeping-mvmt-costs} as the regret of the choices $\zt$ against the losses $\hellt$ plus a weighted movement term.
  Here, $\hellt\in\R^N$ has components $\hellti=\alpha_{ti}(\ell_{ti}-\inner{\ell_t,\pt}-\lambda_{t}\norm{\pt-\ptmm}_1)$. Indeed, we have
  \begin{align*}
    R_{T}^{\alpha}(\cmp,\lambda_{1:T})
    &:=
      \sumtT \sum_{i\in[N]}\alpha_{ti}\cmp_{i}\brac{\inner{\ell_{t},\pt}-\ell_{ti}+\lambda_t\norm{\pt-\ptmm}_{1}}\\
    &\overset{(a)}{=}
      \sumtT -\inner{\hellt, \cmp}\\
    &\overset{(b)}{=}
      \sumtT\sbrac{\inner{\hellt,\zt-\cmp}
      +\inner{\zt,\alpha_t}\lambda_t\norm{\pt-\ptmm}_1}
  \end{align*}
  
  where $(a)$ is by definition of $\hellt$, and $(b)$ adds and subtracts 
  $\inner{\hellt,\zt}$ and observes that
  \begin{align*}
    \inner{\hellt,\zt}
    &=\sum_{i}\alpha_{ti}(\ell_{ti}-\inner{\ell_{t},\pt}-\lambda_{t}\norm{\pt-\ptmm}_{1})\zti\\
      &= \sum_{i}\alpha_{ti}\ell_{ti}\zti - \inner{\ell_{t},\pt}\brac{\sum_{i}\alpha_{ti}\zti}-\norm{\pt-\ptmm}_1\sum_{i\in[N]}\alpha_\ti\lambda_t\zti\\
    &=\sum_{i}\alpha_{ti}\ell_{ti}\zti - \inner{\ell_{t},\frac{\zt^\alpha}{\norm{\zt^\alpha}_{1}}}\brac{\sum_{i}\alpha_{ti}\zti}-\lambda_{t}\norm{\pt-\ptmm}_{1}\inner{\alpha_{t},\zt}\\
    &=\sum_{i}\alpha_{ti}\ell_{ti}\zti - \sum_{i}\alpha_{ti}\ell_{ti}\zti-\lambda_{t}\norm{\pt-\ptmm}_{1}\inner{\alpha_{t},\zt}=-\lambda_{t}\norm{\pt-\ptmm}_{1}\inner{\alpha_{t},\zt}.
  \end{align*}
  Note that the above is true even if 
  $\norm{\zt^\alpha}_1=0$, since this can happen 
  only when $\alpha_{ti}\z_{ti}=0$ for all $i\in[N]$, in which case the equality would still hold if we set $\pt$ arbitrarily in this case.
  Use the pre-interval convention $z_0=z_1=0$, $\alpha_0=\alpha_1$, and $p_0=p_1$, and let $s_t=\norm{\zt^\alpha}_1$ for $t=0,\ldots,T$. By the algorithm's definition, $\zt^\alpha=s_t\pt$ even when $s_t=0$. Since $\ptmm\in\Delta_{N-1}$, the triangle inequality and the reverse triangle inequality give
  \begin{align*}
    s_t\norm{\pt-\ptmm}_1
    &=\norm{\zt^\alpha-s_t\ptmm}_1\\
    &\le\norm{\zt^\alpha-\ztmm^\alpha}_1+\abs{s_t-s_{t-1}}\\
    &\le 2\norm{\zt^\alpha-\ztmm^\alpha}_1.
  \end{align*}
  Multiplying by $\lambda_t\ge0$ and summing over $t$ yields
  \[
    \sumtT\lambda_t\norm{\pt-\ptmm}_1\inner{\alpha_t,\zt}
    \le\sumtT 2\lambda_t\norm{\zt^\alpha-\ztmm^\alpha}_1.
  \]
  Continuing the coordinatewise decomposition,
  \begin{align*}
    \sumtT 2\lambda_t\norm{\zt^\alpha-\ztmm^\alpha}_1
      &=
      \sumtT 2\lambda_t\sum_{i\in[N]}\abs{\alpha_{ti}z_{ti}-\alpha_{\tmm,i}z_{\tmm,i}}\\
      &=
     \sum_{i\in[N]}\sumtT\sbrac{2\lambda_t\alpha_\tmmi\alpha_{ti}\abs{\zti-\ztmmi} + 2\lambda_t (1-\alpha_\ti)\alpha_\tmmi\ztmmi},
  \end{align*}
  where the last line observes the following three cases for $\alpha_{ti}$ supported on 
  $[a_i,b_i]$: $(a)$ for $t\in[a_i+1,b_i]$, we have $\alpha_\ti=\alpha_\tmmi$ 
  so $\abs{\alpha_\ti\zti-\alpha_\tmmi\ztmmi}=\alpha_\ti\abs{\zti-\ztmmi}$,
  $(b)$
  on rounds $t\le a_i$, we have $\abs{\alpha_\ti\zti-\alpha_\tmmi\ztmmi}=0$
  since $\zti=0$ for $t\le a_i$, and $(c)$ for $t= b_i+1$ we have
  $\abs{\alpha_\ti\zti-\alpha_\tmmi\ztmmi}=\alpha_\tmmi\ztmmi$ while for $t>b_i+1$ we have $\abs{\alpha_\ti\zti-\alpha_\tmmi\ztmmi}=0$,
  so more generally we have $\abs{\alpha_\ti\zti-\alpha_\tmmi\ztmmi}=(1-\alpha_\ti)\alpha_\tmmi\ztmmi$
  for $t\ge b_i+1$.
  Thus, plugging this back in above,
  \begin{align}
      R_T^\alpha(\cmp,\lambda_{1:T})
      &\le 
      \sumtT \inner{\hellt,\zt-\cmp}\nonumber\\
      &\qquad
      +2\sumtT\sum_{i\in [N]}\Big[\lambda_t\alpha_\tmmi\alpha_{ti}\abs{\zti-\ztmmi} + \lambda_t(1-\alpha_\ti)\alpha_\tmmi\ztmmi\Big]\nonumber\\
    &\le
      \sum_{i\in[N]}\underbrace{\sbrac{\sumtT \ellti^\lambda(\zti-\cmp_i)+\sumtT 2\alpha_{\tmm,i}\alpha_{ti}\lambda_{t}\abs{z_{ti}-z_{\tmm,i}}}}_{\eqqcolon \tilde R_{T}^{(i)}(u_i, \lambda_{1:T}^{(i)})} \nonumber\\
      &\qquad
      + \sum_{i\in[N]}\cmp_i\sumtT 2(1-\alpha_\tppi)\alpha_\ti\lambda_\tpp\label{eq:R-alpha-decomp}
  \end{align}
where we define 
$\lambda_{t}^{(i)}=2\alpha_{\tmmi}\alpha_{ti}\lambda_{t}$ for
brevity and 
\begin{align*}
    \ellti^\lambda
    &=\hellti + (1-\alpha_\tppi)2\lambda_\tpp\alpha_\ti\\
    &= \alpha_\ti(\ellti -\inner{\ell_t,\pt}-\lambda_t\norm{\pt-\ptmm}_1)+2(1-\alpha_\tppi)\alpha_\ti\lambda_\tpp.
\end{align*}
  Observe for any $i\in[N]$ and $t\in[a_i,b_i]$, we have
  \begin{align*}
      \abs{\ellti^\lambda}&\le \alpha_\ti\abs{\inner{\ell_t,e_i-\pt}}+\alpha_\ti\lambda_t\norm{\pt-\ptmm}_1 + 2(1-\alpha_\tppi)\alpha_\ti\lambda_\tpp\\
                         &\le 2\alpha_\ti(\norm{\ell_t}_\infty+\lambda_t+\lambda_\tpp)\le 2(G_{\infty}+2\lambda)
  \end{align*}
  For $L=2G_{\infty}+5\lambda$, on each round we have
  \begin{align*}
    \abs{\ellti^\lambda}+\lambda_{\tpp}^{(i)}
    &\le \alpha_\ti\abs{\inner{\ell_t,e_i-\pt}}+\alpha_\ti\lambda_t\norm{\pt-\ptmm}_1\\
    &\quad+2\alpha_\ti(1-\alpha_\tppi)\lambda_\tpp+2\alpha_\ti\alpha_\tppi\lambda_\tpp\\
    &=\alpha_\ti\brac{\abs{\inner{\ell_t,e_i-\pt}}+\lambda_t\norm{\pt-\ptmm}_1+2\lambda_\tpp}\\
    &\le 2G_{\infty}+4\lambda\le L.
  \end{align*}
  so we can apply \Cref{thm:pf-with-mc} with $L=2G_{\infty}+5\lambda$
  to the losses $x\mapsto \ellti^\lambda x$ and movement coefficient $\lambda_t^{(i)}$ 
  separately in each coordinate $i\in[N]$ to get
  \begin{align*}
     \tilde R_T^{(i)}(\cmp_i, \lambda_{1:T}^{(i)}) 
     &\le 
     \Ohat\Bigg(
    \epsilon_iL+
     u_i\sqrt{V_{T+1,i}^\lambda\Log{u_i \Lambda_{T+1,i}+1}}
    +  L u_i \Log{u_i\Lambda_{T+1,i}+1}
     \Bigg),
  \end{align*}
  where
  \begin{align*}
      V_{T+1,i}^\lambda
      &=
      \cO\brac{L^2+\sumtT(\abs{\ellti^\lambda}+\lambda_{\tpp}^{(i)})^2}
      =
      \cO\brac{L^2+\sumtT\alpha_\ti\sbrac{(\ell_\ti-\inner{\ell_t,\pt})^2+\lambda_t^2}},\\
      \Lambda_{T+1,i} &= \frac{\sqrt{V_{T+1,i}^{\lambda}}\log^{2}\brac{V_{T+1,i}^{\lambda}/L^{2}}}{\epsilon_{i} L}.
  \end{align*}
  Plugging this back into \Cref{eq:R-alpha-decomp} and summing over $i\in[N]$
  yields
  \begin{align*}
      R_T^\alpha(\cmp)
      &\le
      \cO\Bigg(
      \sum_{i\in [N]}
         u_i\sbrac{\sqrt{V_{T+1,i}^\lambda\Log{u_i \Lambda_{T+1,i}+1}}
        \maxOp L \Log{u_i\Lambda_{T+1,i}+1}}\\
         &\qquad
        +L \sum_{i\in[N]}\epsilon_i+ \sum_{i\in[N]}\cmp_i\sumtT (1-\alpha_\tppi)\alpha_\ti\lambda_\tpp
      \Bigg),
  \end{align*}
  and the stated result follows by observing that for 
  $\alpha_\ti$ is supported on $[a_i,b_i]$, it holds that
  $(1-\alpha_\tppi)\alpha_\ti=1$
  only on round $b_i$
  so the last term is $2\sum_i \cmp_i\lambda_{b_i+1}\le 2\lambda=\cO(L)$ for $\cmp\in\Delta_{N-1}$.\qedhere

  \end{proof}

\subsection[Relation to Jun et al. (2017)]{Relation to \citet{jun2017improved}}

\label{app:relationtojun}

In this section we show that our result in \Cref{thm:sleeping-mvmt-costs} recovers the best-known 
result for the standard sleeping experts setting as a special case, 
while also incorporating movement costs and obtaining Lipschitz adaptivity.
In particular, the following result shows that \Cref{thm:sleeping-mvmt-costs}
matches the regret guarantee of CBCE up to lower-order terms \citep[Corollary
1]{jun2017improved}.
  \begin{restatable}{corollary}{SleepingMvmtCostsCor}\label{cor:sleeping-mvmt-costs}
  Let $\pi\in\mathrm\Delta_{N-1}$ be supported on $[N]$ and set $\epsilon_i=\pi_i$ for each $i\in[N]$. 
  Then under the same assumptions as \Cref{thm:sleeping-mvmt-costs}, for any $\cmp\in \Delta_{N-1}$,
  \Cref{alg:sleeping-mvmt-costs} guarantees
  \begin{align*}
    R_{T}^{\alpha}(\cmp,\lambda_{1:T})
    &\le
      \tilde\cO\Bigg(
      \sqrt{\sum_{i\in[N]}\cmp_iV_{T+1,i}^\lambda\sbrac{\KL(\cmp|\pi) + \Log{\frac{\overline V_{T+1}^\lambda(\cmp)}{L^2}+1\Big)}}}
  \end{align*}
  where we define
  $\overline V_{T+1}^\lambda(\cmp) = \sum_i u_i V_{T+1,i}^\lambda$,
  while
  $V_{T+1,i}^\lambda$ and $B_{T+1,i}^\lambda$ are defined as in \Cref{thm:sleeping-mvmt-costs}.
\end{restatable}
\begin{proof}
From \Cref{thm:sleeping-mvmt-costs}, we have 
\begin{align*}
    R_T^\alpha(\cmp,\lambda_{1:T})
    &=
      \cO\Bigg(L\sum_i\pi_i+ \sum_{i\in[N]}\cmp_{i}\sqrt{V_{T+1,i}^\lambda\Log{\cmp_{i}\Lambda_{T+1,i}+1}}\\
      &\qquad
      +L\sum_{i\in[N]}\cmp_i\Log{\cmp_i\Lambda_{T+1,i}+1}\Bigg)\\
    &=
        \cO\Bigg(L + \sqrt{\sum_{i\in[N]}\cmp_iV_{T+1,i}^\lambda\sum_{i\in[N]}\cmp_i\Log{\cmp_i\Lambda_{T+1,i}+1}}\\
      &\qquad
      +L\sum_{i\in[N]}\cmp_i\Log{\cmp_i\Lambda_{T+1,i}+1}\Bigg)
\end{align*}
where 
$\Lambda_{T+1,i}=\tfrac{\sqrt{V_{T+1,i}^\lambda}\log^2(V_{T+1,i}^\lambda/L^2)}{\pi_i}$
and the last line applies the Cauchy-Schwarz inequality.
  Now observe that for any $u, \beta\ge 0$ and $ \epsilon >0$, we can bound
  \begin{align*}
    u\Log{\frac{u\beta}{\epsilon}+1}
    &=
      u\Log{\frac{u}{\epsilon}}+u\Log{\beta+\frac{\epsilon}{u}}
      \le
      u\Log{\frac{u}{\epsilon}} + u\Log{(1+\beta)\brac{1+\frac{\epsilon}{u}}}\\
    &=
      u\Log{\frac{u}{\epsilon}} + u\Log{1+\beta} + u\Log{1+\frac{\epsilon}{u}}\\
    &\le
      u\Log{\frac{u}{\epsilon}} + u\Log{1+\beta} + \epsilon,
  \end{align*}
  where we have used $u\Log{1+\epsilon/u}\le \epsilon$ and
  $\beta+\frac{\epsilon}{u}\le (1+\beta)(1+\frac{\epsilon}{u})$ for $\beta,\epsilon, u> 0$ and observed that 
  the inequality is trivially true when $u=0$ or $\beta=0$. Therefore,
  we have
  \begin{align*}
    &\sum_{i\in[N]}u_{i}\Log{u_{i}\Lambda_{T+1,i}+1} \\
    &\quad\le
      \sum_{i\in[N]}u_{i}\Log{\frac{u_{i}}{\epsilon_{i}}} + \sum_{i\in[N]}u_{i}\Log{\tfrac{\sqrt{V_{T+1,i}^\lambda}\log^2(V_{T+1,i}^\lambda/L^2)}{L}+1}+1\\
    &\quad=
      1+\KL(u|\epsilon) + \sum_{i\in[N]}u_{i}\Log{\tfrac{\sqrt{V_{T+1,i}^\lambda}\log^2(V_{T+1,i}^\lambda/L^2)}{L}+1}\\
      &\quad\le
      1+\KL(u|\epsilon) + \Log{\frac{\sum_i\cmp_i\sqrt{V_{T+1,i}^\lambda}\log^2(V_{T+1,i}^\lambda/L^2)}{L}+1},
  \end{align*}
  where the last line applies Jensen's inequality.
  Hence, plugging this back in above yields 
  \begin{align*}
      R_T^\alpha(\cmp,\lambda_{1:T})
      &=
      \hat \cO\Bigg(
      L+ 
      \sqrt{\sum_{i\in[N]}\cmp_iV_{T+1,i}^\lambda\sbrac{1+\KL(\cmp|\pi) + \Log{\tfrac{\sum_{i\in[N]}\cmp_i\sqrt{V_{T+1,i}^\lambda}}{L}+1}}}\\
      &\qquad
      +L\sbrac{1+\KL(u|\pi)+\Log{\tfrac{\sum_{i\in[N]}\cmp_i\sqrt{V_{T+1,i}^\lambda}}{L}+1}}
      \Bigg)\\
      &=
      \tilde\cO\brac{
      \sqrt{\sum_{i\in[N]}\cmp_iV_{T+1,i}^\lambda\sbrac{1+\KL(\cmp|\pi) + \Log{\frac{\sum_{i\in[N]}\cmp_i\sqrt{V_{T+1,i}^\lambda}}{L}+1}}}
      },
  \end{align*}
  Applying Jensen's inequality
  to bound $\Log{\sum_iu_i \sqrt{V_{T+1,i}^\lambda}}\le \Log{\sqrt{\sum_i\cmp_i V_{T+1,i}^\lambda}}=\half\Log{\overline{V}_{T+1}^\lambda(\cmp)}$
  yields the stated guarantee.
\end{proof}

%% file: sections/appendix-pfmc.tex
\section{Parameter-free OLO with Movement Costs}\label{app:pf-with-mc}
In this section we provide an extension of the parameter-free algorithm 
of \citet{jacobsen2022parameter} to handle movement costs. The 
result is a straightforward extension of the original analysis by
using the 
movement-costs-via-stability analysis of \citet{qiu2026parameterfree}.
The result is provided for completeness.

\begin{algorithm}[ht]
    \textbf{Input:}~value $L,\epsilon>0$\\
    \textbf{Initialize:}~$\w_1=\zeros\in\ww$\\
    \textbf{Define:} 
    \[\psi(w;V,\beta)=3\int_0^{\norm{w}}\min_{0<\eta\le \tfrac{1}{L}}\sbrac{\frac{\Log{x/\beta+1}}{\eta}+\eta V} \diff x\]
    \For{$t=1,\ldots,T$}{
        Set $V_t^\lambda = 4L^{2} + \sum_{\tau=1}^\tmm(\norm{\gtau}+\lambda_{\tau+1})^2$ and
        $\beta_{t}=\frac{\epsilon L}{\sqrt{V_t^\lambda}\log^2(V_t^\lambda/L^{2})}$\\
        \hfill\\
        Play $\wt =\argmin_{w\in\ww}\inner{\sum_{\tau=1}^\tmm\gtau,\w}+\psi(w; V_t^\lambda,\beta_t)$\\
        Observe $\gt\in\partial\ft(\wt)$ and movement coefficient $\lambda_{\tpp}$
    }
\caption{Parameter-free FTRL for OLO with Movement Costs}
\label{alg:pf-with-mc}
\end{algorithm}

\begin{restatable}{theorem}{PFWithMC}\label{thm:pf-with-mc}
   Let \Cref{asm:lipschitz} hold.
  Then
  for all $\cmp\in\ww$,
  \Cref{alg:pf-with-mc} with $L\ge \norm{\gt}+\lambda_{\tpp}$ for all $t$ guarantees
  \begin{align*}
    R_{T}(\cmp,\lambda_{1:T})
    &\!\le
      \cO\Bigg(\!\epsilon L\!+\! \norm{\cmp}\Bigg[\sqrt{V_{T+1}^\lambda\Log{\norm{\cmp}\Lambda_{T+1}+1}}
      \maxOp L\Log{\norm{\cmp}\Lambda_{T+1}+1}\!\Bigg]\!\Bigg),
  \end{align*}
  where
  $\Lambda_{T+1} = \tfrac{\sqrt{V_{T+1}^\lambda}\log^2(V_{T+1}^\lambda/L^{2})}{\epsilon L}$.
\end{restatable}
\begin{proof}

By convexity of $\ft$, we have that 
\begin{align*}
    R_T(\cmp,\lambda_{1:T})=\sumtT\ft(\wt)-\ft(\cmp) + \lambda_t\norm{\wt-\wtmm}
    &\le 
    \sumtT \inner{\gt, \wt-\cmp} + \lambda_t\norm{\wt-\wtmm},
\end{align*}
where $\gt\in\partial\ft(\wt)$. Now applying the standard FTRL regret template to the first term 
(see, \textit{e.g.}, \citealp[Lemma 7.1]{orabona2019modern}),
we have
\begin{align*}
   R_T(\cmp,\lambda_{1:T}) 
   &\le 
   \psi_{T+1}(\cmp) + \sumtT F_t(\wt)-F_\tpp(\wtpp)+\inner{\gt,\wt} + \sumtT \lambda_t\norm{\wt-\wtmm},
\end{align*}
where $F_t(w)=\inner{\sum_{\tau=1}^\tmm\gtau,w}+\psi_t(w)$ and $\psi_t(w) = \psi(w; V_t^\lambda,\alpha_t)$. Moreover, since $\wt \in \argmin_{w\in\ww}F_t(w)$, we have
\begin{align*}
    &\sumtT F_t(\wt)-F_\tpp(\wtpp)+\inner{\gt,\wt} \\
    &\quad=
    \sumtT F_t(\wt)-F_t(\wtpp) + (\psi_t-\psi_\tpp)(\wtpp) +\inner{\gt,\wt-\wtpp}\\
    &\quad\overset{(a)}{=}
    \sumtT \inner{\grad F_t(\wt)+\gt, \wt-\wtpp}-D_{F_t}(\wtpp|\wt)+ (\psi_t-\psi_\tpp)(\wtpp) \\
    &\quad\overset{(b)}{\le}
    \sumtT \inner{\gt, \wt-\wtpp}-D_{\psi_t}(\wtpp|\wt)+ (\psi_t-\psi_\tpp)(\wtpp),
\end{align*}
where $(a)$ uses the definition of Bregman divergence to write $f(x)-f(y) = \inner{\grad f(x),x-y}-D_f(y|x)$ and $(b)$ uses the first-order optimality condition for $\wt\in\argmin_{\w\in\ww}F_t(w)$
to bound $\inner{\grad F_t(\wt),\wt-\wtpp}\le 0$ and the fact that Bregman divergences are invariant
to linear terms, so $D_{F_t}(x|y)=D_{\psi_t}(x|y)$ for any $x,y\in\ww$. Plugging this back into the 
regret bound and re-indexing the movement costs, we have
\begin{align*}
    R_T(\cmp,\lambda_{1:T})
    &\le 
    \psi_{T+1}(\cmp) + \sumtT \inner{\gt+\lambda_\tpp\frac{\wt-\wtpp}{\norm{\wt-\wtpp}},\wt-\wtpp}-D_{\psi_t}(\wtpp|\wt) \\
    &\qquad + \sumtT (\psi_t-\psi_\tpp)(\wtpp)\\
    &\le
    \psi_{T+1}(\cmp) + \sumtT \underbrace{(\norm{\gt}+\lambda_\tpp)\norm{\wt-\wtpp}-D_{\psi_t}(\wtpp|\wt) + (\psi_t-\psi_\tpp)(\wtpp)}_{=:\delta_t} \;.
\end{align*}
For $\norm{\gt}+\lambda_{\tpp}\le L$ for all $t$, the latter term can be bound borrowing the stability analysis of \citet[Theorem~6]{jacobsen2022parameter}:
\begin{align*}
     \sumtT \delta_t
     &\le
     \sumtT \frac{2\alpha_t}{\sqrt{V_t^\lambda}}(\norm{\gt}+\lambda_\tpp)^2
       =
     2\epsilon L\sumtT \frac{(\norm{\gt}+\lambda_\tpp)^2}{V_t^\lambda\log^{2}(V_{t}^{\lambda}/L^{2})}
     \le
     4\epsilon L\ ,
\end{align*}
where the last line applies \Cref{lemma:tuning-sum}  to the scalar sequence $(\norm{\gt}+\lambda_\tpp)_{t=1}^T$ with $V_t=V_t^\lambda$.
Finally, letting  $f(x)=\min_{0<\eta\le 1/L}\sbrac{\frac{\Log{x/\alpha+1}}{\eta}+\eta V_{T+1}^\lambda}$,
we have that 
\begin{align*}
    \psi_{T+1}(\cmp)
    &=
    3\int_0^{\norm{\cmp}} f(t) \diff t
    \le
    3\norm{\cmp}f(\norm{\cmp}) \\
    &\le
    3\norm{\cmp}\sbrac{\sqrt{V_{T+1}^\lambda\Log{\frac{\norm{\cmp}}{\alpha_{T+1}}+1}} \lor L\Log{\frac{\norm{\cmp}}{\alpha_{T+1}}+1}} \;.
\end{align*}

The stated bound follows by hiding doubly logarithmic factors and observing that $V_{T+1}^\lambda=\cO\brac{\sumtT (\norm{\gt}+\lambda_\tpp)^2}=\cO\brac{\sumtT \norm{\gt}^2+\lambda_t^2}$.
\end{proof}

We state also for completeness the analogous result to the above achieving the ideal 
$\sumtT \norm{\gt}^2+\lambda_t\norm{\gt}$ dependence, which results by applying the blocking reduction of \citet[Algorithm 3]{qiu2026parameterfree}
with \Cref{alg:pf-with-mc} as the base algorithm. We do not use this result explicitly 
here, but it is of independent interest since the result of \citet{qiu2026parameterfree}
obtains the optimal $\tilde\cO\big(\norm{\cmp}\sqrt{\sum_t \norm{\gt}^2+\norm{\gt}\lambda_t}\big)$
\emph{static} regret
only as a special-case of the optimal \emph{dynamic} regret guarantee, which requires
$\cO(d\log(T))$ per-round computation to obtain. The following result 
provides the optimal static regret guarantee using the usual $\cO(d)$
per-round computation. The 
proof is immediate 
by \citet[Lemma B.1]{qiu2026parameterfree}, and the
fact that $\norm{\tgtau}+\tlambda_{\taupp}\le G+2\lambda$, 
by definition of the blocking reduction. 
\begin{restatable}{theorem}{OptimalStaticMC}\label{thm:optimal-static-mc}
   Let \Cref{asm:lipschitz} hold and set $L\ge G+2\lambda$. 
  Then
  for all $\cmp\in\ww$, applying \Cref{alg: dyn_para_meta_2} with \Cref{alg:pf-with-mc} guarantees
  \begin{align*}
      R_T(\cmp,\lambda_{1:T})
      &\le
      \cO\brac{ L\epsilon + \norm{\cmp}\sbrac{\sqrt{V_{T+1}^\lambda\Log{\norm{\cmp}\Lambda_{T+1}+1}}\maxOp  L\Log{\norm{\cmp}\Lambda_{T+1}+1}}}
  \end{align*}
  where  $\hat\cO(\cdot)$ hides constants and doubly-logarithmic factors and
  \begin{align*}
      \Lambda_{T+1}&=\tfrac{\sqrt{V_{T+1}^\lambda}\log^2(V_{T+1}^\lambda/L^{2})}{\epsilon L},\qquad
      V_{T+1}^\lambda=\cO\brac{\sumtT \norm{\gt}^2+\norm{\gt}\lambda_t}.
  \end{align*}
\end{restatable}

%% file: sections/appendix-adaptive-gd.tex
\section{Adaptive Gradient Descent with Movement Costs}%
\label{app:adaptive-gd}

In this section, we present a simple baseline algorithm that achieves a 
dynamic regret bound of $\mathcal{O}(
(D + P_{\ab}) \sqrt{\sum_{t\in\ab}(\|g_t\|^2 + \lambda_{\tpp}^2)})
$, which can be improved to $\cO\brac{(D+P_\ab)\sqrt{\sum_{t\in\ab}\norm{\gt}^2+\lambda_t\norm{\gt}}}$ using the blocking reduction of 
\Cref{alg: dyn_para_meta_2}. Note that since we will obtain this bound on all intervals $\ab$ simultaneously, it can be shown that this implies the optimal $\sqrt{(D^2+DP_{\ab})\sum_{t\in\ab}( \norm{\gt}^2+\lambda_\tpp\norm{\gt}})$ on 
each sub-interval simultaneously as well (\citet[Theorem 7]{cutkosky2020parameter}). 
\pref{alg:adaptive-gd} is based on Online Gradient Descent (OGD)~\citep{zinkevich2003online} with an adaptive learning rate choice.

\begin{algorithm}
  \SetAlgoLined
    \textbf{Input:} Convex domain $\ww\subset\R^{d}$, parameter $\alpha>0$\\
    \textbf{Initialize:} $w_{1}\in\ww$\\
    \For{$t=1,\ldots,T$}{
        Play $\wt\in \ww$, observe $\gt\in\partial\ft(\wt)$ and $\lambda_{\tpp}$
        
        Set $\eta_{\tpp}=\frac{\alpha}{\sqrt{\sum_{\tau=1}^{t}(\norm{\gtau}+\lambda_\taupp)^2}}$ 
        
        Update $\wtpp = \Pi_{\ww}\brac{\wt - \eta_{\tpp}\gt}$\\
    }
  \caption{Adaptive Gradient Descent with Time-varying Movement Costs}
  \label{alg:adaptive-gd}
\end{algorithm}
The following theorem provides the regret of the movement-cost-adaptive gradient descent procedure.
\begin{restatable}{theorem}{AdaptiveGD}\label{thm:adaptive-gd}
  For any sequence $\cmp_{1:T}=(\cmp_1,\ldots,\cmp_T)\in\ww^T$, \Cref{alg:adaptive-gd} guarantees
  \begin{align*}
    R_{T}(\cmp_{1:T},\lambda_{1:T})
    &\le
      \sbrac{\frac{D^2+DP_{T}}{\alpha}+\alpha}\sqrt{2\sumtT \norm{\gt}^2+\lambda_{\tpp}^2} \;,
  \end{align*}
  where $P_T=\sum_{t=2}^T\norm{\cmp_t-\cmp_\tmm}$.
\end{restatable}
\begin{proof}
  The algorithm can be interpreted as an instance of mirror descent with
  $\psi_{\tpp}(w)=\frac{1}{2\eta_{\tpp}}\norm{\w}^{2}$, and
  $\wtpp = \argmin_{\w\in\ww}\inner{\gt,\w}+D_{\psi_{\tpp}}(\w|\wt)$.
  By the first-order optimality condition, we know that for any $w\in \ww$
  \begin{align*}
      \inner{g_t+\nabla\psi_{t+1}(w_{t+1})-\nabla\psi_{t+1}(w_t),w-\wtpp} \geq 0.
  \end{align*}
  Picking $w$ to be $u_t$ and re-arranging, we have
  $\inner{\gt,\wtpp-\cmp_{t}}\le \inner{\grad\psi_{\tpp}(\wt)-\grad\psi_{\tpp}(\wtpp),\wtpp-\cmp_{t}}$. Therefore,
  for $\gt\in\partial\ell_t(\wt)$ for all $t$, we have that
  \begin{align*}
    &R_{T}(\cmp_{1:T},\lambda_{1:T}) \\
    &\quad\le
      \sumtT \inner{\gt,\wt-\cmp_{t}}+\lambda_{t}\norm{\wt-\wtmm} \tag{\pref{asm:lipschitz}}\\
    &\quad=
      \sumtT \inner{\gt,\wtpp-\cmp_{t}}+\sumtT \inner{\gt,\wt-\wtpp}+\sumtT\lambda_{t}\norm{\wt-\wtmm}\\
    &\quad\le
      \sumtT \inner{\grad\psi_{\tpp}(\wt)-\grad\psi_{\tpp}(\wtpp),\wtpp-\cmp_{t}}\\
      &\qquad+\sumtT \inner{\gt,\wt-\wtpp}+\sumtT\lambda_{t}\norm{\wt-\wtmm}\\
    &\quad\overset{(a)}{\le}
      \sumtT \left(D_{\psi_{\tpp}}(\cmp_{t}|\wt)-D_{\psi_{\tpp}}(\cmp_{t}|\wtpp)\right) \\
      &\qquad+\sumtT \left((\norm{\gt}+\lambda_{\tpp})\norm{\wt-\wtpp}-D_{\psi_{\tpp}}(\wtpp|\wt)\right)\\
    &\quad\overset{(b)}\le
      \sumtT \left(D_{\psi_{\tpp}}(\cmp_{t}|\wt)-D_{\psi_{\tpp}}(\cmp_{t}|\wtpp)\right) +\sumtT \frac{\eta_\tpp}{2}(\norm{\gt}+\lambda_{\tpp})^2\norm{\gt})\\
    &\quad\overset{(c)}\le
      \sumtT \left(D_{\psi_{\tpp}}(\cmp_{t}|\wt)-D_{\psi_{\tpp}}(\cmp_{t}|\wtpp)\right) + \alpha \sqrt{\sumtT \norm{\gt}^2+\lambda_\tpp^2}
  \end{align*}
  where $(a)$ uses the equality
  $\inner{\grad f(a)-\grad f(b),b-\cmp}=D_{\psi}(\cmp|a)-D_{\psi}(\cmp|b)-D_{\psi}(b|a)$
  and re-indexes the movement costs (recalling $\lambda_{T+1}=\lambda_{1}=0$); $(b)$ uses $D_{\psi_\tpp}(x|y)=\frac{1}{2\eta_\tpp}\norm{x-y}^2$ and applies Fenchel-Young inequality, and $(c)$ uses the definition of $\eta_{t+1}$ and the fact that,
\begin{align*}
\sum_{t=1}^T
\frac{\eta_{t+1}\bigl(\|g_t\|+\lambda_{t+1}\bigr)^2}
{2\sqrt{\sum_{\tau=1}^t
\bigl(\|g_\tau\|+\lambda_{\tau+1}\bigr)^2}} 
&\leq
\alpha \sum_{t=1}^T
\frac{(\|g_t\|+\lambda_{t+1})^2}
{2\sqrt{\sum_{\tau=1}^t
\bigl(\|g_\tau\|+\lambda_{\tau+1}\bigr)^2}} \\
&\leq
\alpha \sum_{t=1}^T
\left(
\sqrt{\sum_{\tau=1}^t
\bigl(\|g_\tau\|+\lambda_{\tau+1}\bigr)^2}
-
\sqrt{\sum_{\tau=1}^{t-1}
\bigl(\|g_\tau\|+\lambda_{\tau+1}\bigr)^2}
\right) \\
&\leq
\alpha
\sqrt{
\sum_{t=1}^T
\bigl(\|g_t\|+\lambda_{t+1}\bigr)^2}\;.
\end{align*}

For the first summation of the above equation, direct calculation shows that
  \begin{align*}
    &\sumtT \left(D_{\psi_{\tpp}}(\cmp_{t}|\wt)-D_{\psi_{\tpp}}(\cmp_{t}|\wtpp)\right) \\
    &=
      \sumtT \left( D_{\psi_{t}}(\cmp_{t}|\wt)-D_{\psi_{\tpp}}(\cmp_{t}|\wtpp)\right)
      +\sumtT \left(D_{\psi_{\tpp}}(\cmp_{t}|\wt)-D_{\psi_{t}}(\cmp_{t}|\wt)\right)\\
    &=
      D_{\psi_{1}}(\cmp_{1}|\w_{1})-D_{\psi_{T+1}}(\cmp_{T}|w_{T+1})+\sum_{t=2}^{T} \left(D_{\psi_{t}}(\cmp_{t}|\wt)-D_{\psi_{t}}(\cmp_{\tmm}|\wt)\right)\\
    &\qquad
      +\sumtT \brac{\frac{1}{2\eta_{\tpp}}-\frac{1}{2\eta_{t}}}\norm{\wt-\cmp_{t}}^{2}\\
    &{\le}
      \sum_{t=2}^{T}\psi_{t}(\cmp_{t})-\psi_{t}(\cmp_{\tmm})-\inner{\grad\psi_{t}(\wt),\cmp_{t}-\cmp_{\tmm}}
      +\frac{D^{2}}{\alpha}\sqrt{\sumtT (\norm{\gt}+\lambda_{\tpp})} \\
    &\le
      \sum_{t=2}^{T}\inner{\grad\psi_{t}(\cmp_{t})-\grad\psi_{t}(\wt),\cmp_{t}-\cmp_{\tmm}}
      +\frac{D^{2}}{\alpha}\sqrt{\sumtT (\norm{\gt}+\lambda_{\tpp})^2}\\
    &\le
      \frac{DP_{T}}{\alpha}\sqrt{ \sumtT (\norm{\gt}+\lambda_{\tpp})^2}
      +\frac{D^{2}}{\alpha}\sqrt{\sumtT (\norm{\gt}+\lambda_{\tpp})^2}\\
    &\le
    \frac{D^{2}+DP_{T}}{\alpha}\sqrt{\sumtT (\norm{\gt}+\lambda_{\tpp})^2}.
  \end{align*}
  Hence, overall we have
  \begin{align*}
    R_{T}(\cmp_{1:T},\lambda_{1:T})
    &\le \sbrac{\frac{D^2+DP_{T}}{\alpha}+\alpha}\sqrt{\sumtT (\norm{\gt}+\lambda_{\tpp})^2}\\
    &\le \sbrac{\frac{D^2+DP_{T}}{\alpha}+\alpha}\sqrt{2\sumtT \norm{\gt}^2+\lambda_{\tpp}^2},
  \end{align*} 
  where the last line uses the elementary inequality
  $(a+b)^2\le 2a^2+2b^2$.
\end{proof}

%% file: sections/appendix-first-order-redux.tex
\section{Adaptive First-order Reduction}
\label{app:firstorderreduction}
In this section, we provide a modest refinement of 
the first-order reduction of \citet{qiu2026parameterfree}
which 
allows us to convert $\sum_t\norm{\gt}^2+\lambda_t^2$ dependencies into the 
better $\sum_t\norm{\gt}^2+\lambda_t\norm{\gt}$ dependencies over any \emph{sub-interval} $\ab$. 
The reduction is shown in \Cref{alg: dyn_para_meta_2}, and the following lemma provides the key property of the aggregate losses induced by the reduction.
The main difference from the reduction of \citet{qiu2026parameterfree} is we use a blocking condition based on $\sum_t\norm{\gt}$ rather than 
$\norm{\sum_t\gt}$. This is useful because the latter condition can make bounding losses over partial epochs difficult, which 
becomes relevant in the strongly-adaptive setting, where an interval $[a,b]$ may begin mid-way through a given epoch.

\begin{algorithm}[ht]
    \caption{Adaptive First-order Reduction}
    \label{alg: dyn_para_meta_2}
        \textbf{Initialize: } Online learning algorithm $\widetilde\cA$. Denote $\widetilde{w}_1$ to be $\widetilde \calA$'s first decision, $\tau = 1$, $\tg_\tau=\zeros$, $H_\tau=\mathbf{0}$. %
    
        \For{$t = 1, 2, \dots$}{
            Play $w_t = \widetilde w_{\tau}$, then receive $g_t$ and $\lambda_{t+1}$
            
            Update $\tgtau \leftarrow \tgtau+ g_t $ and $H_\tau \leftarrow H_{\tau}+\norm{\gt}$ 
            
            \If{$ H_{\tau} > \lambda_{t+1}$}{
            
            Set $\widetilde\lambda_{\tau+1} = \lambda_{t+1}$
            
            Send $(\widetilde g_{\tau}, \widetilde\lambda_{\tau+1})$ to $\widetilde \calA$ , and receive $\widetilde w_{\tau+1}$
            
            Set  $H_{\tau+1}=0$, $\tg_\taupp = \zeros$   
            
            $\tau \gets \tau + 1$
            }
        }
\end{algorithm}

Let $I_1,\ldots,I_n$ be the nonempty epochs partitioning $[T]$, with the final epoch truncated at $T$, and write $I_\tau=\Set{s_\tau,\ldots,e_\tau}$ and $\tgtau=\sum_{t\in I_\tau}\gt$.
We use the conventions $\tlambda_1=0$ and $\tw_0=\tw_1$.
If the final epoch has not triggered an update by round $T$, we complete it for analysis by passing $(\tg_n,0)$ to the base learner after round $T$, so that $\tlambda_{n+1}=0$.
This terminal update does not change any decision on $[T]$.
If an update was already triggered on round $T$, we retain the outgoing coefficient assigned by the algorithm.

The following provides a modest  generalization of the blocking argument 
of \citet{qiu2026parameterfree} to account for the fact that
an arbitrary $\ab$ may not align with the blocks defined by \Cref{alg: dyn_para_meta_2}.
\SABlocking*
\begin{proof}
Let $I_1,\ldots,I_n$ be the epochs defined by the execution of \Cref{alg: dyn_para_meta_2} and denote $I_\tau=\Set{s_\tau,\ldots,e_\tau}$ for each $\tau\in[n]$. 
Given an arbitrary $[a,b]$, we first
decompose the interval into the set of complete epochs. 
Let 
\begin{align*}
    \cE_\ab = \Set{\tau: I_\tau\subseteq\ab}
\end{align*}
be the set of complete epochs contained within $\ab$,
and denote the set of partial epochs as
\[
    \cP_\ab = \Set{\tau\in[n]: I_\tau\cap\ab\ne\varnothing,\ I_\tau\nsubseteq\ab}.
\]
Since every such epoch contains $a$ or $b$, we have $\abs{\cP_\ab}\le 2$.
Then by convexity of $\ft$, we have
\begin{align*}
    \sumab\ft(\wt)-\ft(\cmp_t)
    &\le\sumab\inner{\gt,\wt-\cmp_t}\\
    &=
    \sum_{\tau\in\cE_\ab}\sum_{t\in I_\tau}\inner{\gt,\wt-\cmp_t}
    +\sum_{\tau\in \cP_\ab}\sum_{t\in I_\tau\cap\ab}\inner{\gt,\wt-\cmp_t}.
\end{align*}
The partial epochs can each be bound
using the blocking condition:
\begin{align*}
    \sum_{\tau\in\cP_\ab}\sum_{t\in I_\tau\cap\ab}\inner{\gt,\wt-\cmp_t}
    &\le 
    \sum_{\tau\in\cP_\ab}D\sum_{t\in I_\tau\cap\ab}\norm{\gt}\\
    &\le
    \sum_{\tau\in\cP_\ab}D\brac{G+\sum_{t=s_\tau}^{\min\Set{e_\tau-1,b-1}}\norm{\gt}}
    \le 2D(G+\lambda_{\max}^\ab),
\end{align*}
where $\lambda_{\max}^\ab=\max_{t\in\ab}\lambda_t$.
Hence, overall we can bound
\begin{align}
    \sum_{t\in\ab}\inner{\gt,\wt-\cmp_t}
    &\le 
    2D(G+\lambda_{\max}^{\ab}) + \sum_{\tau\in\cE_\ab}\sum_{t\in I_\tau}\inner{\gt,\wt-\cmp_t}.\label{eq:sa-blocking-1}
\end{align}
Focusing on the second term, 
choose an arbitrary epoch-level comparator 
$\widetilde{\cmp}_\tau\in \Set{u_t:t\in I_\tau}$
for each $\tau\in\cE_\ab$. We can then write
\begin{align}
    \sum_{\tau\in\cE_\ab} \sum_{t\in I_{\tau}}\inner{\gt,\wt-\cmp_t}\nonumber
    &=
    \sum_{\tau\in\cE_\ab} \sum_{t\in I_\tau} \inner{\gt, \twtau-\tcmp_\tau} + \sum_{\tau\in\cE_\ab}\sum_{t\in I_\tau}\inner{\gt, \tcmp_\tau-\cmp_t}\nonumber\\
    &=
    \sum_{\tau\in\cE_\ab} \inner{\tgtau, \twtau-\tcmp_\tau} + \sum_{\tau\in\cE_\ab}\sum_{t\in I_\tau}\inner{\gt, \tcmp_\tau-\cmp_t}\nonumber\\
    &=
    \sum_{\tau\in\cE_\ab} \inner{\tgtau, \twtau-\tcmp_\tau} + \underbrace{\sum_{\tau\in\cE_\ab}\sum_{t\in I_\tau}\inner{\gt, \cmp_{e_\tau}-\cmp_t}}_{\eqqcolon A}\nonumber\\
    &\qquad
    + \underbrace{\sum_{\tau\in\cE_\ab}\sum_{t\in I_\tau}\inner{\gt, \tcmp_\tau-\cmp_{e_{\tau}}}}_{\eqqcolon B},\label{eq:sa-blocking-2}
\end{align}
where in the first line we've used the fact that $\wt=\twtau$ for $t\in I_\tau$ and the last line adds and subtracts $\cmp_{e_\tau}$, the 
last comparator in each epoch $\tau$. To bound $A$ observe that
\begin{align*}
    A 
    &= 
    \sum_{\tau\in\cE_\ab} \sum_{t\in I_\tau}\inner{\gt, \sum_{s=\tpp}^{e_\tau}\cmp_{s}-\cmp_{s-1}}\\
    &=
    \sum_{\tau\in\cE_\ab} \sum_{t=s_{\tau}}^{e_{\tau}}\sum_{s=\tpp}^{e_\tau}\inner{\gt, \cmp_{s}-\cmp_{s-1}}\\
    &=
    \sum_{\tau\in\cE_\ab} \sum_{s=s_\tau+1}^{e_{\tau}}\sum_{t=s_{\tau}}^{s-1}\inner{\gt, \cmp_{s}-\cmp_{s-1}}\\
    &\le
    \sum_{\tau\in\cE_\ab} \sum_{s=s_\tau+1}^{e_{\tau}}\brac{\sum_{t=s_{\tau}}^{s-1}\norm{\gt}} \norm{\cmp_{s}-\cmp_{s-1}}\\
    &\le
    \sum_{\tau\in\cE_\ab} \sum_{s=s_\tau+1}^{e_{\tau}}\lambda_{s}\norm{\cmp_{s}-\cmp_{s-1}}
    \le \lambda_{\max}^\ab P_{[a,b]}
\end{align*}
where the last line observes that $\sum_{t=s_\tau}^{s-1}\norm{g_t}\le \lambda_s$ for any $s\in [s_\tau,e_\tau]$ 
and over-approximates $\sum_{\tau\in\cE_\ab}P_{I_\tau}\le P_\ab$. The term B can be bounded by 
\begin{align*}
   B
   &=
    \sum_{\tau\in\cE_\ab}\sum_{t\in I_\tau}\inner{\gt, \tcmp_\tau-\cmp_{e_{\tau}}}\\
    &=
    \sum_{\tau\in\cE_\ab}\inner{\tgtau, \tcmp_\tau-\cmp_{e_{\tau}}}\\
    &\le
    \sum_{\tau\in\cE_\ab}\norm{\tgtau}\norm{\tcmp_\tau-\cmp_{e_{\tau}}}\\
    &\le
    \sum_{\tau\in\cE_\ab}(G+\lambda_{e_\tau})\norm{\tcmp_\tau-\cmp_{e_{\tau}}}
    \le
    (G+\lambda_{\max}^\ab)P_\ab
\end{align*}
where the last line uses that $\norm{\tgtau} \le G+\lambda_{e_\tau}$ by definition of the blocking condition 
and uses triangle inequality to bound $\norm{\tcmp_\tau-\cmp_{e_\tau}}\le P_{I_\tau}$ for any $\tau$, and then 
over-approximates $\sum_{\tau\in\cE_\ab}P_{I_\tau}\le P_\ab$. Therefore, plugging the bounds for $A$ and $B$ back
into \Cref{eq:sa-blocking-2}, we have
\begin{align*}
    \sum_{\tau\in\cE_\ab}\sum_{t\in I_\tau}\inner{\gt,\wt-\cmp_t}
    &\le 
    \sum_{\tau\in\cE_\ab}\inner{\tgtau,\twtau-\tcmp_\tau} + (G+2\lambda_{\max}^\ab) P_\ab,
\end{align*}
so overall, combining with \Cref{eq:sa-blocking-1} we have
\begin{align*}
    \sumab \ft(\wt)-\ft(\cmp_t) \le 2(D+P_\ab)(G+\lambda_{\max}^\ab)+\sum_{\tau\in\cE_\ab}\inner{\tgtau,\twtau-\tcmp_\tau}.
\end{align*}
The stated result follows by adding in the movement costs into the bound and
observing that
$\sum_{t=a+1}^{b}\lambda_t\norm{\wt-\wtmm}\le\sum_{\tau\in\cE_\ab}\tlambda_\tau\norm{\twtau-\tw_{\taumm}}+D\lambda_{\max}^{\ab}$
to bound
\begin{align*}
    &\sumab \ft(\wt)-\ft(\cmp_t) +\sum_{t=a+1}^{b}\lambda_t\norm{\wt-\wtmm}\\
    &\qquad\le 3(D+P_\ab)(G+\lambda_{\max}^\ab)
    +\sum_{\tau\in\cE_\ab}\Brb{\inner{\tgtau,\twtau-\tcmp_\tau}
    +\tlambda_\tau\norm{\twtau-\tw_{\taumm}}}.
\end{align*}
\end{proof}

We will also need the following key bound. The result shows that batching based on $\sum_{t}\norm{\gt}$ 
rather than $\norm{\sum_t\gt}$ preserves the key second-order-to-first-order property that
drives our approach.
\begin{lemma}
\label{lemma:blocking-gradient-sum}
    For every $\tau \in [n]$, \Cref{alg: dyn_para_meta_2} guarantees that
    \begin{equation*}
        \norm{\widetilde{g}_{\tau}}^2 + \widetilde{\lambda}_{\tau+1}^2
        \le \sum_{t\in I_\tau} \Brb{2\norm{g_{t}}^2 + 4\lambda_{t} \norm{g_t}} \;.
    \end{equation*}
\end{lemma}
\begin{proof}
  For every epoch that triggers an update, the blocking condition gives
  $\tlambda_{\tau+1}=\lambda_{e_\tau+1}<\sum_{t\in I_\tau}\norm{\gt}$ (and for an unfinished final epoch, our terminal convention gives $\tlambda_{n+1}=0$).
  Thus $\tlambda_{\tau+1}\le\sum_{t\in I_\tau}\norm{\gt}$ for every $\tau\in[n]$.
  Similarly, by triangle inequality,
  $\norm{\tgtau}\le \sum_{t\in I_\tau}\norm{\gt}$, 
  so both $\norm{\tgtau}^2$ and $\tlambda_{\taupp}^2$ can be 
  bound in terms of $\brac{\sum_{t\in I_\tau}\norm{\gt}}^2$.
  Hence,
  \begin{align*}
      \norm{\tgtau}^2+\tilde\lambda_{\tau+1}^2
      &\le
      2\brac{\sum_{t\in I_\tau}\norm{\gt}}^2 
      = 
      2\sum_{t\in I_\tau}\norm{\gt}^2 + 2\sum_{t\in I_\tau}\sum_{s\in I_\tau, s\ne t}\norm{\gt}\norm{\gs}\\
      &=
      2\sum_{t\in I_\tau}\norm{\gt}^2 + 4\sum_{t\in I_\tau}\norm{\gt}\sum_{s\in I_\tau, s<t}\norm{\gs}\\
      &\le
      2\sum_{t\in I_\tau}\norm{\gt}^2 + 4\sum_{t\in I_\tau}\lambda_t\norm{\gt}\;.\qedhere
  \end{align*}
\end{proof}

The following lemma provides a bound on the movement coefficients 
and aggregate losses incurred via the blocking reduction of \Cref{alg: dyn_para_meta_2}.
\begin{lemma}\label{lemma:blocking-gradient-bound}
Let $\tau\in[n]$ be an arbitrary epoch defined by the execution of \Cref{alg: dyn_para_meta_2},
and let $I_{\tau}=[s_\tau,e_\tau]$ be the active rounds of epoch $\tau$. Then
    \begin{align*}
        \tlambda_{\tau+1}\le \sum_{t\in I_\tau}\norm{\gt}\le \lambda_{e_\tau} + \norm{g_{e_\tau}}
    \end{align*}
\end{lemma}
\begin{proof}
    The fact that 
    $\tlambda_{\tau+1}\le \sum_{t\in I_\tau}\norm{\gt}$
    follows from the blocking condition on each epoch $I_\tau$ that triggers an update, and the condition holds trivially on an unfinished final epoch
    since $\tlambda_{n+1}=0$
    by our terminal convention in that case.
    Moreover,
    \begin{align*}
        \sum_{s\in I_\tau}\norm{\gs}
        =\sum_{s=s_\tau}^{e_\tau-1}\norm{\gs}+\norm{g_{e_\tau}}
        \le \lambda_{e_\tau}+\norm{g_{e_\tau}}.
    \end{align*}
    Indeed, if $e_\tau>s_\tau$, the epoch did not end on round $e_\tau-1$, so the blocking condition implies
    $\sum_{s=s_\tau}^{e_\tau-1}\norm{\gs}\le\lambda_{e_\tau}$ before processing round $e_\tau$.
    If $e_\tau=s_\tau$, this sum is empty and the inequality follows from $\lambda_{e_\tau}\ge0$.

\end{proof}

Finally, the following lemma observes that
under the blocking reduction, for any $[a,b]$ we can always assume that $\lambda_t\le b G$ without loss of generality. 
\begin{restatable}{lemma}{MCBound}\label{lemma:mc-bound}
    Under the blocking reduction of \Cref{alg: dyn_para_meta_2}, we may assume without loss of generality that $\lambda_t\le G(t-1)$ for any $t$.
\end{restatable}
\begin{proof}
    By construction \Cref{alg: dyn_para_meta_2} only plays on rounds where $\lambda_{t}\le \sum_{t\in I_\tau}\norm{\gt}$ for some epoch $\tau$, and we only pay the movement costs associated with rounds where we transition to the next epoch. But on round $t$, 
    the largest possible value of the RHS is $\sum_{s=1}^\tmm\norm{\gs}\le (t-1)G$. Therefore, any round $t$ with 
    $\lambda_t> G(t-1)$ can never have triggered a new epoch, so on these rounds we have $\wt=\wtmm$ and a movement cost of $\lambda_t\norm{\wt-\wtmm} = 0$. Therefore under the blocking reduction, we may assume without loss of generality that 
    no such events occur.
\end{proof}

%% file: sections/appendix-main.tex
\section{Proofs for Section~\ref{sec:main} (\SecMain)}%
\label{app:main}

In this section we prove our main result. The
theorem is restated below for convenience.
\Main*
\begin{proof}
Fix any interval $\ab\subseteq[T]$, comparator sequence
$\cmp_a,\ldots,\cmp_b\in\ww$, and movement coefficients
$\lambda_a,\ldots,\lambda_b\in[0,\lambda]$.
If $a=b$, then $R_\ab(\cmp_{a:b},\lambda_{a:b})\le DG$,
which implies the claimed bound. Hence, suppose that $a<b$.

If $P_\ab<D$, take $K=1$ and $I_1=\ab$.
Otherwise, \Cref{lemma:path-length-partition} gives a partition
$I_1,\ldots,I_K$ of $\ab$.
In both cases, $P_{I_k}\le 2D$ for every $k\in[K]$ and
$K\le 1+P_\ab/D$.
Denote the intervals in this partition by $I_k=[s_k,e_k]$ in chronological order, so that $s_k=e_{k-1}+1$ for $2\le k\le K$.
Decomposing the regret over these intervals gives
\begin{align*}
    R_\ab(\cmp_{a:b},\lambda_{a:b})
    &=
    \sum_{k=1}^K
    R_{I_k}(\cmp_{s_k:e_k},\lambda_{s_k:e_k})
    +\sum_{k=2}^K
    \lambda_{s_k}\norm{w_{s_k}-w_{s_k-1}}\\
    &\le
    \sum_{k=1}^K
    R_{I_k}(\cmp_{s_k:e_k},\lambda_{s_k:e_k})
    +(K-1)D\lambda.
\end{align*}
Applying \Cref{prop:main-any-ab} to each interval, using
$P_{I_k}\le 2D$ and $e_k\le b$, and then applying
the Cauchy--Schwarz inequality, we obtain
\begin{align*}
    R_\ab(\cmp_{a:b},\lambda_{a:b})
    &\le
    \hat\cO\Bigg(
    KD(G+\lambda)\log^2(b+1)\\
    &\qquad
    +D\sqrt{\log(b+1)}
    \min\Bigg\{
    \sum_{k=1}^K\sqrt{G(G+\lambda)\abs{I_k}},
    \sum_{k=1}^K\sqrt{V_{I_k}^\lambda\log(b+1)}
    \Bigg\}
    \Bigg)\\
    &\le
    \hat\cO\Bigg(
    KD(G+\lambda)\log^2(b+1)\\
    &\qquad
    +D\sqrt{K}
    \min\Bigg\{
    \sqrt{G(G+\lambda)(b-a+1)\log(b+1)},
    \sqrt{V_\ab^\lambda}\log(b+1)
    \Bigg\}
    \Bigg).
\end{align*}
Substituting $K\le 1+P_\ab/D$ and using
$b-a+1\le 2(b-a)$, we conclude that
\begin{align*}
    &R_\ab(\cmp_{a:b},\lambda_{a:b})\\
    &\le
    \hat\cO\Bigg(
    (D+P_\ab)(G+\lambda)\log^2(b+1)\\
    &\qquad
    +\sqrt{D^2+DP_\ab}
    \min\Bigg\{
    \sqrt{(G^2+G\lambda)\abs{b-a}\log(b+1)},
    \sqrt{V_\ab^\lambda}\log(b+1)
    \Bigg\}
    \Bigg),
\end{align*}
which proves the theorem.
\end{proof}

The following proposition bounds the regret on a geometric covering interval.
\begin{restatable}{proposition}{MainGC}\label{prop:main-gc}
  Denote by $\cI$ the set of geometric covering intervals.
  Under the assumptions of \Cref{thm:main}, for any $\ab\in\cI$,
  \pref{alg:ol-mvmt-costs} guarantees
  \begin{align*}
     R_\ab(\cmp_{a:b},\lambda_{a:b})
     &\le
      \hat\cO\brac{
      (D\Omega_\ab+P_\ab)(G+\lambda)
      +\brac{D\Omega_\ab^\half+P_\ab}
      \sqrt{\sumab(\norm{\gt}^2+\lambda_t\norm{\gt})}},
  \end{align*}
  where $\Omega_\ab\triangleq\log(b+1)$ is the logarithmic factor
  used in the interval bounds below, and $\hat\cO(\cdot)$ hides
  constants and doubly-logarithmic factors.
\end{restatable}
\begin{proof}
\Cref{alg:ol-mvmt-costs} is constructed by applying the blocking
reduction of \Cref{alg: dyn_para_meta_2} followed by the
expert/meta-regret decomposition in \Cref{lemma:mc-redux}.
We begin by applying the blocking reduction decomposition of
\Cref{lemma:sa-blocking}.
Denote by $n$ the total number of epochs defined during the execution
of \Cref{alg: dyn_para_meta_2}, and denote the rounds of epoch $\tau$
by $I_\tau=[s_\tau,e_\tau]$.
Recall that $\tgtau=\sum_{t\in I_\tau}\gt$, and choose
$\tcmp_\tau=\cmp_{s_\tau}$ as the epoch-level comparator for $I_\tau$.
Denote by
$\cE_\ab\triangleq\Set{\tau\in[n]:I_\tau\subseteq\ab}$
the set of complete epochs contained within $\ab$.
If $\cE_\ab=\emptyset$, then \Cref{lemma:sa-blocking} gives
$R_\ab(\cmp_{a:b},\lambda_{a:b})
\le 3(D+P_\ab)(G+2\lambda_{\max}^\ab)$,
which implies the claimed bound.
Otherwise, denote its first and last epochs by $a'$ and $b'$, so that
$\cE_\ab=\Set{a',\ldots,b'}$.
Applying \Cref{lemma:sa-blocking}, we have
\begin{align}
    R_\ab(\cmp_{a:b},\lambda_{a:b})
    &=
    \sumab\bigl(\ft(\wt)-\ft(\cmp_t)\bigr)
    +\sum_{t=a+1}^b\lambda_t\norm{\wt-\wtmm}\nonumber\\
    &\le
    3(D+P_\ab)(G+2\lambda_{\max}^\ab)
    +\underbrace{
    \sum_{\tau=a'}^{b'}
    \brac{
    \inner{\tgtau,\twtau-\tcmp_\tau}
    +\tlambda_\tau\norm{\twtau-\twtaumm}
    }}_{S_{[a',b']}}\nonumber\\
    &\le
    4(D+P_\ab)(G+2\lambda_{\max}^\ab)
    +\tilde R_{[a',b']}
    (\tcmp_{a':b'},\tlambda_{a':b'}),
    \label{eq:main:step-1}
\end{align}
where $S_{[a',b']}$ denotes the displayed sum over complete epochs.
The final inequality includes the movement entering the first
complete epoch $a'$, which appears in $S_{[a',b']}$ but is excluded
from the interval regret $\tilde R_{[a',b']}$ by the convention
of \Cref{sec:setting}. Indeed,
\begin{align*}
    S_\abp
    &=
    \tilde R_\abp(\tcmp_{a':b'},\tlambda_{a':b'})
    +\tlambda_{a'}\norm{\tw_{a'}-\tw_{a'-1}}\\
    &\le
    \tilde R_\abp(\tcmp_{a':b'},\tlambda_{a':b'})
    +D\lambda_{\max}^\ab,
\end{align*}
where we used $\tlambda_{a'}=\lambda_{s_{a'}}
\le\lambda_{\max}^\ab$.

Next, we bound the regret over the full epochs,
$\tilde R_\abp(\tcmp_{a':b'},\tlambda_{a':b'})$.
Recall that $\twtau=\sum_{I\in\cI}\tptauI\twtau^{(I)}$,
where $\twtau^{(I)}=w^{(I)}_{s_\tau}$ and
$\tptauI=p_{s_\tau,I}$ are the outputs of $\cA_I$ and
$\cA_\meta$ on epoch $\tau$.
The specialist $\cA_J$ associated with the geometric covering
interval $J=[a,b]$ is active on epoch $\tau$ precisely when
$s_\tau\in J$.
Denote by $c$ its last active epoch.
Its activation interval is $[a',c]$, where
$c\in\{b',b'+1\}$; the case $c=b'+1$ corresponds to an
additional epoch that begins within $J$ but extends beyond $b$.

Suppose first that $c=b'+1$.
Applying \Cref{lemma:sa-mc-redux} over the specialist's full
activation interval $[a',c]$, with the full epoch-level horizon $n$,
and subtracting the contribution of epoch $c$ from both sides yields
\begin{align*}
    \widetilde R_{[a',b']}
    (\widetilde u_{a':b'},\widetilde\lambda_{a':b'})
    &\le
    \widetilde R^{\cA_J}_{[a',b']}
    (\widetilde u_{a':b'},\widetilde\lambda_{a':b'})
    +\widetilde R_n^\alpha
    (e_J,\widetilde\lambda_{1:n}^{\meta})
    +\Delta_c,
\end{align*}
where $\tlambda_\tau^\meta=D\tlambda_\tau$ is the
meta-level movement coefficient, and the correction for the
additional epoch is
\begin{align*}
    \Delta_c
    \triangleq
    \inner{\tg_c,\tw_c^{(J)}-\tw_c}
    +\tlambda_c\brac{
    \norm{\tw_c^{(J)}-\tw_{c-1}^{(J)}}
    -\norm{\tw_c-\tw_{c-1}}}\le
    D(\norm{\tg_c}+\tlambda_c)
    \le D(G+2\lambda),
\end{align*}
with the last inequality following from
\Cref{lemma:blocking-gradient-bound}.
When $c=b'$, we apply \Cref{lemma:sa-mc-redux} directly
without this correction. Thus, in either case,
\begin{align}
    \widetilde R_{[a',b']}
    (\widetilde u_{a':b'},\widetilde\lambda_{a':b'})
    &\le
    D(G+2\lambda)
    +\widetilde R^{\cA_J}_{[a',b']}
    (\widetilde u_{a':b'},\widetilde\lambda_{a':b'})
    +\widetilde R_n^\alpha
    (e_J,\widetilde\lambda_{1:n}^{\meta}).
    \label{eq:main:step-2}
\end{align}

The regret of the base learner is bounded by
\Cref{thm:adaptive-gd} with $\alpha=D$:
\begin{align}
    \tilde R_\abp^{\cA_J}(\tcmp_{a':b'},\tlambda_{a':b'})
    &\le
    \brac{
    2D+\sum_{\tau=a'+1}^{b'}
    \norm{\tcmp_\tau-\tcmp_\taumm}}
    \sqrt{
    2\sum_{\tau=a'}^{b'}
    \brac{\norm{\tgtau}^2+\tlambda_\taupp^2}}
    \nonumber\\
    &\le
    2(2D+P_\ab)
    \sqrt{\sumab(\norm{\gt}^2+2\lambda_t\norm{\gt})},
    \label{eq:main:adaptive-gd}
\end{align}
where the last line applies \Cref{lemma:blocking-gradient-sum}
and uses the triangle inequality to bound
$\sum_{\tau=a'+1}^{b'}\norm{\tcmp_\tau-\tcmp_{\tau-1}}
\le P_\ab$.

To bound the meta-regret, denote by
$G_\infty^\meta\triangleq D(G+2\lambda)$ and
$\lambda^\meta\triangleq D\lambda$
the bounds on the meta-losses and movement coefficients.
By \Cref{lemma:blocking-gradient-bound}, we have
\begin{align*}
    \norm{\telltau}_\infty
    \le
    \max_{I\in\cI}
    \abs{
    \inner{\tgtau,\twtau^{(I)}}
    +\tlambda_\tau
    \norm{\twtau^{(I)}-\tw_{\taumm}^{(I)}}}\le
    D(\norm{\tgtau}+\tlambda_\tau)
    \le G_\infty^\meta.
\end{align*}
Moreover, $\tlambda_\taupp^\meta\le D\lambda=\lambda^\meta$.
Hence, setting
$L=2G_\infty^\meta+5\lambda^\meta=D(2G+9\lambda)$
satisfies the conditions of \Cref{thm:sleeping-mvmt-costs}.
Applying \Cref{eq:sleeping-explicit-bound} with $\cmp=e_J$ yields
\begin{align*}
    \tilde R_n^\alpha(e_J,\tlambda_{1:n}^\meta)
    &\le
    \hat\cO\brac{
    \Gamma(\epsilon)
    +\sqrt{
    V_{n+1,J}^\lambda
    \Log{\frac{\sqrt{V_{n+1,J}^\lambda}}{\epsilon_J L}+1}}
    +L\Log{\frac{\sqrt{V_{n+1,J}^\lambda}}{\epsilon_J L}+1}},
\end{align*}
where we recall that
\begin{align*}
    \Gamma(\epsilon)
    &=L\brac{1+\sum_{I\in\cI}\epsilon_I},\\
    V_{n+1,J}^\lambda
    &=
    4L^2+\sum_{\tau=1}^n
    \brac{\abs{\tell_{\tau,J}^\lambda}
    +\tlambda_{\taupp}^{(J)}}^2,
\end{align*}
and
\begin{align*}
    \tell_{\tau,J}^\lambda
    &=
    \talpha_{\tau,J}\brac{
    \inner{\tell_\tau,e_J-\tp_\tau}
    -\tlambda_\tau^\meta\norm{\tptau-\tptaumm}_1
    +(1-\talpha_{\taupp,J})2\tlambda_{\taupp}^\meta},\\
    \tlambda_\taupp^{(J)}
    &=
    2\talpha_{\tau,J}\talpha_{\taupp,J}
    \tlambda_\taupp^\meta
    \le 2D\tlambda_{\taupp}.
\end{align*}
Using $\norm{\telltau}_\infty
\le D(\norm{\tgtau}+\tlambda_\tau)$
and $L=D(2G+9\lambda)$, we obtain
\begin{align}
    V^\lambda_{n+1,J}
    &=
    4L^2+\sum_{\tau=1}^n
    \brac{\abs{\tell^\lambda_{\tau,J}}
    +\tlambda^{(J)}_{\tau+1}}^2\nonumber\\
    &\le
    \cO\brac{
    L^2+D^2\sum_{\tau=a'}^c
    \brac{\norm{\tg_\tau}^2
    +\tlambda_\tau^2+\tlambda_{\tau+1}^2}}
    \nonumber\\
    &\le
    \cO\brac{
    L^2+D^2\sum_{\tau=a'}^{b'}
    \brac{\norm{\tg_\tau}^2+\tlambda_{\tau+1}^2}}
    \nonumber\\
    &\le
    \cO\brac{
    L^2+D^2\sum_{t=a}^b
    \brac{\norm{\gt}^2+\lambda_t\norm{\gt}}},
    \label{eq:V-bound}
\end{align}
where the second inequality uses
$D^2\sum_{\tau=a'}^c\tlambda_\tau^2
=\cO(L^2+D^2\sum_{\tau=a'}^{b'}\tlambda_{\taupp}^2)$.
When $c=b'+1$, it also uses
$D^2(\norm{\tg_c}^2+\tlambda_{c+1}^2)=\cO(L^2)$
from \Cref{lemma:blocking-gradient-bound}.
The last inequality applies \Cref{lemma:blocking-gradient-sum}.

Choosing $\epsilon_\ab=1/(a^2(1+\lfloor\log_2 a\rfloor))$
for each $[a,b]\in\cI$, we have
\begin{align*}
    \Gamma(\epsilon)
    =
    L\brac{
    1+\sum_{I=[s,e]\in\cI}
    \frac{1}{s^2(1+\lfloor\log_2s\rfloor)}}\le
    L\brac{1+\sum_{s=1}^{\infty}\frac{1}{s^2}}
    \le 3L,
\end{align*}
since at most $1+\lfloor\log_2s\rfloor$ geometric covering
intervals start at $s$.
Furthermore, \Cref{eq:V-bound} and the choice of $\epsilon_J$ give
$\Log{\frac{\sqrt{V_{n+1,J}^\lambda}}{L\epsilon_J}+1}
=\cO(\log(b+1))=\cO(\Omega_\ab)$
for $J=[a,b]$.
Substituting these bounds into the meta-regret guarantee yields
\begin{align}
    \tilde R_n^\alpha(e_J,\tlambda_{1:n}^\meta)
    &\le
    \hat\cO\Bigg(
    L+D\sqrt{
    \sumab(\norm{\gt}^2+\lambda_t\norm{\gt})
    \Log{\frac{\sqrt{V_{n+1,J}^\lambda}}{L\epsilon_J}+1}}
    \nonumber\\
    &\qquad
    +L\Log{\frac{\sqrt{V_{n+1,J}^\lambda}}{L\epsilon_J}+1}
    \Bigg)\nonumber\\
    &\le
    \hat\cO\Bigg(
    L+D\sqrt{
    \Omega_\ab\sumab(\norm{\gt}^2+\lambda_t\norm{\gt})}
    +L\Omega_\ab
    \Bigg).
    \label{eq:main:meta-bound}
\end{align}
Substituting \Cref{eq:main:adaptive-gd,eq:main:meta-bound}
into \Cref{eq:main:step-2} gives
\begin{align*}
    \tilde R_{[a',b']}(\tcmp_{a':b'},\tlambda_{a':b'})
    &\le
    \hat\cO\Bigg(
    L\Omega_\ab
    +\brac{D\Omega_\ab^\half+P_\ab}
    \sqrt{\sumab(\norm{\gt}^2+\lambda_t\norm{\gt})}
    \Bigg).
\end{align*}
Substituting this into \Cref{eq:main:step-1}, we conclude that
\begin{align*}
    R_\ab(\cmp_{a:b},\lambda_{a:b})
    &\le
    4(D+P_\ab)(G+2\lambda_{\max}^\ab)\\
    &\qquad
    +\hat\cO\Bigg(
    L\Omega_\ab
    +\brac{D\Omega_\ab^\half+P_\ab}
    \sqrt{\sumab(\norm{\gt}^2+\lambda_t\norm{\gt})}
    \Bigg)\\
    &\le
    \hat\cO\brac{
    (D\Omega_\ab+P_\ab)(G+\lambda)
    +\brac{D\Omega_\ab^\half+P_\ab}
    \sqrt{\sumab(\norm{\gt}^2+\lambda_t\norm{\gt})}},
\end{align*}
where we used $L=\cO(D(G+\lambda))$ and
$\lambda_{\max}^\ab\le\lambda$.
\end{proof}

The next proposition extends \Cref{prop:main-gc}
to an arbitrary interval $\ab\subseteq[T]$.
\begin{restatable}{proposition}{MainAnyAB}\label{prop:main-any-ab}
Under the same conditions as \Cref{prop:main-gc}, for any
$\ab\subseteq[T]$, we have
\begin{align*}
    &R_\ab(\cmp_{a:b},\lambda_{a:b})\\
    &\le
    \hat\cO\Bigg(
    \brac{D\Omega_\ab\log(b-a+2)+P_\ab}(G+\lambda)\\
    &\qquad
    +\brac{D\Omega_\ab^\half+P_\ab}
    \min\Bigg\{
    \sqrt{G(G+\lambda)(b-a+1)},
    \sqrt{\log(b-a+2)
    \sumab(\norm{\gt}^2+\lambda_t\norm{\gt})}
    \Bigg\}
    \Bigg),
\end{align*}
where $\Omega_\ab=\log(b+1)$.
\end{restatable}
\begin{proof}
By \Cref{lemma:gc-partition}, there is a partition of $\ab$
into disjoint geometric covering intervals
$I_{-p},\ldots,I_0,I_1,\ldots,I_q$ such that
$\abs{I_{-i}}/\abs{I_{-i+1}}\le\half$
for $i=1,\ldots,p$ and
$\abs{I_i}/\abs{I_{i-1}}\le\half$
for $i=2,\ldots,q$.
In particular, $p+q+1=\cO(\log(b-a+2))$.
Denote each interval by $I_k=[s_k,e_k]$.
Decomposing the regret gives
\begin{align*}
    R_\ab(\cmp_{a:b},\lambda_{a:b})
    &=
    \sum_{k=-p}^q
    R_{I_k}(\cmp_{s_k:e_k},\lambda_{s_k:e_k})
    +\sum_{k=-p+1}^q
    \lambda_{s_k}\norm{w_{s_k}-w_{s_k-1}}\\
    &\le
    \sum_{k=-p}^q
    R_{I_k}(\cmp_{s_k:e_k},\lambda_{s_k:e_k})
    +D\lambda(p+q).
\end{align*}
Applying \Cref{prop:main-gc} to each interval and using
$\Omega_{I_k}\le\Omega_\ab$ and
$\sum_{k=-p}^qP_{I_k}\le P_\ab$, we obtain
\begin{align*}
    R_\ab(\cmp_{a:b},\lambda_{a:b})
    &\le
    \hat\cO\Bigg(
    D\lambda\log(b-a+2)
    +\sum_{k=-p}^q
    (D\Omega_{I_k}+P_{I_k})(G+\lambda)\\
    &\qquad
    +\sum_{k=-p}^q
    (D\Omega_{I_k}^\half+P_{I_k})
    \sqrt{\sum_{t\in I_k}
    (\norm{\gt}^2+\lambda_t\norm{\gt})}
    \Bigg)\\
    &\le
    \hat\cO\Bigg(
    \brac{D\Omega_\ab\log(b-a+2)+P_\ab}(G+\lambda)\\
    &\qquad
    +(D\Omega_\ab^\half+P_\ab)
    \sum_{k=-p}^q
    \sqrt{\sum_{t\in I_k}
    (\norm{\gt}^2+\lambda_t\norm{\gt})}
    \Bigg).
\end{align*}
We bound the remaining sum in two ways.
By the Cauchy--Schwarz inequality,
\begin{align*}
    \sum_{k=-p}^q
    \sqrt{\sum_{t\in I_k}
    (\norm{\gt}^2+\lambda_t\norm{\gt})}
    &\le
    \sqrt{(p+q+1)
    \sum_{k=-p}^q\sum_{t\in I_k}
    (\norm{\gt}^2+\lambda_t\norm{\gt})}\\
    &\le
    \cO\Bigg(
    \sqrt{\log(b-a+2)
    \sumab(\norm{\gt}^2+\lambda_t\norm{\gt})}
    \Bigg).
\end{align*}
Alternatively, the geometric decay of the interval lengths gives
\begin{align*}
    \sum_{k=-p}^q
    \sqrt{\sum_{t\in I_k}
    (\norm{\gt}^2+\lambda_t\norm{\gt})}
    &\le
    \sqrt{G(G+\lambda)}
    \brac{
    \sum_{k=-p}^0\sqrt{\abs{I_k}}
    +\sum_{k=1}^q\sqrt{\abs{I_k}}}\\
    &\le
    2\sqrt{G(G+\lambda)}
    \sum_{i=0}^\infty
    \brac{2^{-i}(b-a+1)}^\half\\
    &\le
    8\sqrt{G(G+\lambda)(b-a+1)}.
\end{align*}
Combining these bounds, we conclude that
\begin{align*}
    &R_\ab(\cmp_{a:b},\lambda_{a:b})\\
    &\le
    \hat\cO\Bigg(
    \brac{D\Omega_\ab\log(b-a+2)+P_\ab}(G+\lambda)\\
    &\qquad
    +(D\Omega_\ab^\half+P_\ab)
    \min\Bigg\{
    \sqrt{G(G+\lambda)(b-a+1)},
    \sqrt{\log(b-a+2)
    \sumab(\norm{\gt}^2+\lambda_t\norm{\gt})}
    \Bigg\}
    \Bigg).
\end{align*}\qedhere
\end{proof}

%% file: sections/appendix-applications.tex
\section{Proofs for Section~\ref{sec:applications} (\SecApplications)}

\subsection{Details for Delayed Feedback} \label{apx:delays}

To facilitate the presentation of our results, we begin by introducing some standard notation for OCO with delayed feedback.
For each round \(t\), let \(o_t \defeq \{\tau \in \mathbb{Z}_+ : \tau+d_\tau < t\} \subseteq [t-1]\) denote the set of past rounds whose gradients have already been observed before round \(t\).
We further define
\(m_t \defeq [t-1]\setminus o_t\) as the set of past rounds whose gradients are still outstanding before round \(t\).
Under this notation, the delayed-feedback protocol is equivalently described by saying that, at the end of round \(t\), the learner receives the batch of newly available gradients \(\{g_\tau\}_{\tau\in o_{t+1}\setminus o_t}\).
We also define \(\sigma_{\max} \defeq \max_{t\in[T]} |m_t|\) to be the maximum number of outstanding gradients at any round, and \(\dtot \defeq \sum_{t=1}^T d_t\) to be the total delay over the time horizon. We now extend the above notation so as to capture dependencies on the sequence of delays restricted to a certain time interval $[a,b] \subseteq [T]$; this is of course required in order to state our strongly adaptive regret guarantees.
In particular, for every $t \in [T+1]$, we let $o_t^{[a,b]} \defeq o_t \cap [a,b]$, $m_t^{[a,b]} \defeq m_t \cap [a,b]$, and $\sigma_{\max}^{[a,b]} \defeq \max_{t \in [a,b]} \abs{m_{t}^{[a,b]}}$ denote the respective quantities restricted to interval $[a,b]$. We also introduce the interval-restricted delays as $d_t^{[a,b]} \defeq \min\{d_t,b-t\}$ for every $t \in [a,b]$, and consequently define the total delay $\dtot^{[a,b]} \defeq \sum_{t=a}^b d_t^{[a,b]}$ for the same interval. In addition, we introduce the maximum delay $\dmax^{[a,b]} = \max_{t \in [a,b]} d_t$ over interval $[a,b]$, defined instead in terms of the unrestricted delays.

\begin{lemma} \label{lem:delay-sum-interval}
    For any sequence of delays,
    \(
        \sum_{t \in \ab}|m_t| \cdot |o_{t+1}\setminus o_t| \le 2\dtot^{\ab} + 3(\dmax^{[1,a]})^2
    \).
\end{lemma}
\begin{proof}
    Begin by observing that $|m_a| \le \dmax^{[1,a]}$.
    Now, using similar calculations as in the proof of \citet[Lemma~D.1]{qiu2026parameterfree}, we can show that
    \begin{align*}
        \sum_{t\in\ab} |m_t| \cdot |o_{t+1}\setminus o_t|
        &= \sum_{t\in\ab} |m_t| \cdot (1 + |m_t| - |m_{t+1}|) \\
        &= \sum_{t\in\ab} |m_t| + |m_a|^2 - |m_b|\cdot|m_{b+1}| + \sum_{t=a+1}^b (|m_t| - |m_{t-1}|)|m_t| \\
        &\le \sum_{t\in\ab} |m_t| + \sum_{t=a+1}^b (|m_t| - |m_{t-1}|)|m_t| + |m_a|^2 \\
        &\le 2\sum_{t\in\ab} |m_t| + |m_a|^2 \\
        &\le 2\sum_{t\in\ab} |m_t| + (\dmax^{[1,a]})^2 \;,
    \end{align*}
    where the first step applies $|o_{t+1}\setminus o_t| = |m_t| + 1 - |m_{t+1}|$, and the second inequality uses the fact that $|m_t| \le |m_{t-1}| + 1$.
    Focusing on the sum from the first term at the right-hand side, we can bound it from above as
    \begin{align*}
        \sum_{t\in\ab} |m_t|
        &= \sum_{t\in\ab} |m_t^{\ab}| + \sum_{t\in\ab} (|m_t\setminus m_t^{\ab}|) \\
        &= \sum_{t\in\ab} |m_t^{\ab}| + \sum_{t\in\ab} (|m_t\setminus \ab|) \\
        &= \sum_{t\in\ab} |m_t^{\ab}| + \sum_{t\in m_a} (\min\{t+d_t, b\} - a + 1) \\
        &\le \sum_{t\in\ab} |m_t^{\ab}| + \sum_{t\in m_a} d_t \tag{$t \le a-1$ for $t \in m_a$} \\
        &\le \sum_{t\in\ab} |m_t^{\ab}| + |m_a|\dmax^{[1,a]} \\
        &\le \sum_{t\in\ab} |m_t^{\ab}| + (\dmax^{[1,a]})^2 \\
        &= \dtot^{\ab} + (\dmax^{[1,a]})^2 \;.
    \end{align*}
    Finally, combining all the above inequalities, we conclude the proof.
\end{proof}

The following lemma now provides a reduction from delayed feedback to OCO with movement costs. 
The result generalizes the reduction of \citet{qiu2026parameterfree} to hold on 
arbitrary intervals. 
\DelayedRedux*
\begin{proof}
  Let
\(
\mathcal F_{[a,b]}
=
\{t\in[a,b]:t+d_t\le b\}
\)
denote the set of rounds whose feedback is received
by the end of the interval. By convexity and
\Cref{asm:lipschitz}, we have
\begin{align*}
R_{[a,b]}^{\mathrm{del}}(u_{a:b})
&\le
\sum_{t\in[a,b]}
\langle g_t,w_t-u_t\rangle
\\
&\le
\sum_{t\in\mathcal F_{[a,b]}}
\langle g_t,w_t-u_t\rangle
+
GD|m_{b+1}^\ab|,
\end{align*}
where $m_{b+1}^\ab=m_{b+1}\cap[a,b]$.
Indeed, the second inequality follows because
$[a,b]\setminus\mathcal F_{[a,b]}
=m_{b+1}^\ab$.
Now, for each $t\in\mathcal F_{[a,b]}$,
the gradient $g_t$ is revealed by round $b$,
so applying \citet[Lemma 5.1]{qiu2026parameterfree}
over the time range $[a,b]$ yields
\begin{align*}
\sum_{t\in\mathcal F_{[a,b]}}
\langle g_t,w_t-u_t\rangle
&\le
\sum_{t=a}^b
\langle\widetilde g_t,w_t-u_t\rangle
+
G\sum_{t=a+1}^b
|m_t^{[a,b]}|
\|w_t-w_{t-1}\|
+
G\sigma_{\max}^{[a,b]}P_{[a,b]},
\end{align*}
where
\(
\widetilde g_t
=
\sum_{\tau\in
o_{t+1}^{[a,b]}\setminus o_t^{[a,b]}}
g_\tau
\)
and $\sigma_{\max}^\ab=\max_{t\in\ab}\abs{m_t^\ab}$.
Hence, combining previous two displays and 
using $|m_t^{[a,b]}|\le |m_t|$ gives
\begin{align*}
R_{[a,b]}^{\mathrm{del}}(u_{a:b})
&\le
\sum_{t=a}^b
\langle\widetilde g_t,w_t-u_t\rangle
+
G\sum_{t=a+1}^b
|m_t|\|w_t-w_{t-1}\|
\\
&\quad+
G\sigma_{\max}^{[a,b]}P_{[a,b]}
+
GD|m_{b+1}^\ab|\\
&\le
D\sum_{t=a}^b\norm{\tgt-h_t}+\sum_{t=a}^b
\langle h_t,w_t-u_t\rangle
+
G\sum_{t=a+1}^b
|m_t|\|w_t-w_{t-1}\|
\\
&\quad+
G(D+P_\ab)\sigma_{\max}^{[a,b+1]}
\end{align*}
  Now, let $\cT_\ab = \{t <a: a\le t+d_t \le b\}=m_a\setminus m_{b+1}$ be the set of rounds before the beginning of the interval $[a,b]$ such that their arrival time falls inside that same interval.
  For each $t$, also define $\cF_t = o_{t+1}\setminus o_t$ and $\cF_t^\ab = o_{t+1}^\ab\setminus o_t^\ab$.
  Observe that $\cT_\ab = \bigcup_{t \in [a,b]} \bigl(\cF_t \setminus \cF_t^\ab\bigr)$ and thus, by the triangle inequality,
  \begin{align*}
      \sum_{t=a}^b \norm{\gtilde_t-h_t}
      &= \sum_{t=a}^b \norm*{\sum_{\tau \in \cF_t \setminus \cF_t^\ab} g_\tau}
      \le \sum_{t=a}^b \sum_{\tau \in \cF_t \setminus \cF_t^\ab} \norm{g_\tau}\\
      &= \sum_{t \in \cT_\ab} \norm{g_t}
      \le G\abs{\cT_\ab} 
    =G\abs{m_a\setminus m_{b+1}}\;.
  \end{align*}
  Then, plugging the above inequalities into the first math display yields the claim:
  \begin{align}
R_{[a,b]}^{\mathrm{del}}(u_{a:b})
&\le
\sum_{t=a}^b
\langle h_t,w_t-u_t\rangle
+
G\sum_{t=a+1}^b
|m_t|\|w_t-w_{t-1}\|\nonumber
\\
&\quad+
G(D+P_\ab)\sigma_{\max}^{[a,b+1]}
+
GD\abs{m_{a}\setminus m_{b+1}}.
  \end{align}
  The stated result follows by bounding $\sigma_{\max}^{a,b+1}\le \sigma$ and $\abs{m_a\setminus m_{b+1}}\le \abs{m_a}\le \sigma$.
\end{proof}

We finally have all the main ingredients to prove the strongly adaptive regret upper bound for OCO with delayed feedback.
\begin{manualtheorem}{\ref{thm:main-delay}}
    Let \Cref{asm:lipschitz} hold and let $H \ge G\max_{t}|o_{t+1}\setminus o_t|$.
    Initializing \Cref{alg:ol-mvmt-costs} from \Cref{thm:main} with Lipschitz constant $H$, and feeding it gradients $h_t = \sum_{\tau \in o_{t+1} \setminus o_t} g_\tau$ and movement coefficients $\lambda_t = G|m_t|$,
    guarantees for any $\ab\subseteq[T]$ and any $\cmp_a,\ldots,\cmp_b \in \ww$ that
    \begin{align*}
        \RD_{\ab}(u_{a:b})
        &\le \hat\cO\Biggl(
            (D + P_{\ab})\Log{b+1}^2L_{\sigma,a}
        \\
            &\qquad\quad+ G\sqrt{(D^2+DP_\ab)\brb{|b-a| + \dtot^{\ab}}} \Log{b+1}
        \Biggr) \;,\qedhere
    \end{align*}
    where $L_{\sigma,a} = H \lor G \brb{\sigma + \dmax^{[1,a]}}$ and $P_\ab=\sum_{t=a+1}^b\norm{\cmp_t-\cmp_\tmm}$. 
\end{manualtheorem}
\begin{proof}[Proof of \Cref{thm:main-delay}]
    We first apply the delays-to-movement-costs reduction of \Cref{lemma:delayed-redux}, which yields 
    \[
        \RD_{\ab}(u_{a:b})
        \le
         G(2D+P_\ab)\sigma
         + \underbrace{\sum_{t\in\ab} \inner{h_t, \wt-\cmp_{t}} + \sum_{t=a+1}^b \lambda_t\norm{\wt-\wtmm}}_{\eqqcolon \tilde R_{\ab}(u_{a:b},\lambda_{a:b})}
    \]
    where we recall that $h_t = \sum_{\tau \in o_{t+1}\setminus o_t} g_\tau$ and $\lambda_t = G|m_t|$.
    Note that each vector $h_t$ satisfies
    \[
        \norm{h_t} \le G|o_{t+1}\setminus o_t| \le H
    \] 
    by the triangle inequality and \Cref{asm:lipschitz}.
    Moreover, observe that $\tilde R_{\ab}(u_{a:b},\lambda_{a:b})$ is the regret over $[a,b]$ for the time-varying movement-costs problem with linear losses given by $h_t$ and movement coefficients given by $\lambda_t=G\abs{\mt}\le G\sigma$.
    Consequently, we can apply \Cref{thm:main} to obtain that, for any interval $[a,b] \subseteq [T]$,
    \begin{align*}
        \tilde R_{\ab}(u_{a:b},\lambda_{a:b})
        &\le \hat \cO\Bigg(
        (D+P_\ab)\log(b+1)^2 (H+G\sigma)\\
        &\qquad\quad
        + \sqrt{(D^2+DP_\ab)\sum_{t\in \ab}\norm{h_t}^2+\lambda_t\norm{h_t}}\Log{b+1}
        \Bigg).
    \end{align*}

    Next, we bound the 
    data-dependent terms $\sumab\norm{h_t}^2+\lambda_t\norm{h_t}$. 
    First, observe that
    \begin{align*}
        \sum_{t\in\ab} |o_{t+1}\setminus o_t| = |o_{b+1} \setminus o_a|
        \le |b-a+1| + |m_a|
        \le |b-a+1| + \dmax^{[1,a]} \;.
    \end{align*}
    Moreover, we can also show that
    \begin{align*}
        \sum_{t \in \ab}|m_t| \cdot |o_{t+1}\setminus o_t| \le 2\dtot^{\ab} + 3(\dmax^{[1,a]})^2
    \end{align*}
    via \Cref{lem:delay-sum-interval}.
    Using both of these facts allows us to derive that
    \begin{align*}
        \sum_{t\in \ab}\norm{h_t}^2+\lambda_t\norm{h_t}
        &\le G^2\sum_{t\in \ab} |o_{t+1}\setminus o_t|\brb{|o_{t+1}\setminus o_t| + |m_t|} \\
        &\le G^2\sum_{t\in \ab} |o_{t+1}\setminus o_t|\brb{2|m_t|+1} \\
        &\le 7G^2 \lrb{|b-a+1| + \dtot^{\ab} + (\dmax^{[1,a]})^2} \;.
    \end{align*}

    Therefore,
    \begin{align*}
        &\sqrt{(D^2+DP_\ab)\sum_{t\in \ab}\norm{h_t}^2+\lambda_t\norm{h_t}} \\
        &\qquad\le \cO\lrb{G\sqrt{(D^2+DP_\ab)\brb{|b-a| + \dtot^{\ab} + (\dmax^{[1,a]})^2}}} \\
        &\qquad= \cO\lrb{G\sqrt{(D^2+DP_\ab)\brb{|b-a| + \dtot^{\ab}}} + G(D + P_{\ab})\dmax^{[1,a]}}.
    \end{align*}

    Combining everything together finally yields
    \begin{align*}
        \RD_{\ab}(u_{a:b})
        &\le \hat\cO\Biggl(
            (D + P_{\ab})\Log{b+1}^2[H\maxOp G(\sigma + d_{\max}^{[1,a]})]
        \\
            &\qquad\quad+ G\sqrt{(D^2+DP_\ab)\brb{|b-a| + \dtot^{\ab}}} \Log{b+1}
        \Biggr) \;,
    \end{align*}
    Defining $L_{\sigma,a}=H\maxOp G(\sigma + d_{\max}^{[1,a]})$
    completes the proof.
\end{proof}

\subsection{Details for Time-Varying Memory} \label{apx:memory}

In this section we provide additional details for our results on OCO with time-varying memory.

The following two lemmas provide inequalities basically proving that the memory dependencies relative to rounds beyond the current interval $\ab$ are sufficiently small and do not largely affect the overall guarantees when fully restricting to $\ab$ only.
\begin{lemma} \label{lem:clipped-memory-error}
    For any sequence of memory sizes and any time interval $\ab \subseteq [T]$,
    \[
        \sum_{t\in\ab} \brb{\msize_t-\msize_t^\ab} \le B_\ab^2 \;.
    \]
\end{lemma}
\begin{proof}
    Let $\cT_{[a,b]} = \lcb{t \in [a,b] : t-\msize_t \le a-1}$ be the set of rounds in the interval $[a,b]$ whose unclipped memory begins strictly before the first round $a$ of the interval.
    Note that, for any $t \in [a,b]$, we have that $\msize_t-\msize_t^{[a,b]} > 0$ if and only if $t \in \cT_{[a,b]}$.
    Furthermore, observe that the cardinality of $\cT_{[a,b]}$ satisfies
    \begin{align*}
        \abs*{\cT_{[a,b]}}
        &= \sum_{t=a}^b \mathbb{I}\{t-\msize_t \le a-1\}
        \le \sum_{t=a}^b \mathbb{I}\lcb{t \le a-1 + B_{[a,b]}}
        \le B_{[a,b]} \;.
    \end{align*}
    Therefore,
    \begin{align*}
        \sum_{t=a}^b \brb{\msize_t-\msize_t^{[a,b]}}
        &= \sum_{t \in \cT_{[a,b]}} \brb{\msize_t-\msize_t^{[a,b]}}
        \le \abs*{\cT_{[a,b]}} B_{[a,b]}
        \le B_{[a,b]}^2 \;.
    \end{align*}
\end{proof}

\begin{lemma} \label{lem:clipped-xi-error}
    For any sequence of memory sizes and any $b \in [T-1]$,
    \[
        \sum_{s=b+1}^{T} \bigl[b_s - (s-b)\bigr]_+ \le B_{[b+1, T]}^2 \;,
    \]
    where $[\cdot]_+ = \max\{\cdot, 0\}$ denotes the positive real part.
\end{lemma}
\begin{proof}
    By direct calculations, we can show that
    \begin{align*}
        \sum_{s=b+1}^{T} \bigl[\msize_s - (s-b)\bigr]_+
        &= \sum_{s=b+1}^{T} \bigl(\msize_s - (s-b)\bigr) \mathbb{I}\{b \ge s-\msize_s\} \\
        &\le \sum_{s=b+1}^{T} \bigl(B_{[b+1,T]} - (s-b)\bigr) \mathbb{I}\{b \ge s-B_{[b+1,T]}\} \\
        &\le \sum_{s=b+1}^{b+B_{[b+1,T]}} \bigl(B_{[b+1,T]} - (s-b)\bigr) \\
        &\le \sum_{i=1}^{B_{[b+1,T]}} \bigl(B_{[b+1,T]} - i\bigr) \\
        &\le B_{[b+1,T]}^2 \;.
    \end{align*}
\end{proof}

The first main result we require is an extension of \citet[Lemma~5.6]{qiu2026parameterfree} to the analogous memory-to-movements reduction over any time interval.
\begin{manuallemma}{\ref{lemma:memory-redux}}
    Let \Cref{asm:coordwise-lipschitz,asm:unary-properties} hold, and let $[a,b] \subseteq [T]$ be any interval.
    Then
    \[
        \RM_{[a,b]}(\cmp_{1:T}) \le GP_{[a,b]}(\bmax^{[a,b]})^2 + 2GDB_{[a,b]}^2 + \sum_{t=a}^b \inner{h_t, \wt-\cmp_t} + \sum_{t=a+1}^b \lambda_t \norm{\wt-\wtmm} \;,
    \]
    where $h_t = \nabla\hat f_t(w_t)$, $\xi_t = \sum_{s=t}^T \mathbb{I}\bcb{t \ge s-\msize_s+1}\brb{t-\brb{s-\msize_s}}$, and $\lambda_t = G\xi_t$.
\end{manuallemma}
\begin{proof}[Proof of \Cref{lemma:memory-redux}]
    For any $t \in [a,b]$, let $\xi_t^\ab \triangleq \sum_{s=t}^b \mathbb{I}\bcb{t \ge s-\msize_s^\ab+1}\brb{t-\brb{s-\msize_s^\ab}}$.
    We then use the fact that \Cref{asm:coordwise-lipschitz} holds to show that
    \begin{align*}
        \abs*{f_t(w_{t-\msize_t:t}) - \hat f_t(\wt)}
        &\le G\sum_{i=0}^{\msize_t} \norm{w_{t}-w_{t-i}} \\
        &\le G\sum_{i=0}^{\msize_t^\ab} \norm{w_{t}-w_{t-i}} + GD\brb{\msize_t-\msize_t^\ab}\\
        &\le G\sum_{i=0}^{\msize_t^\ab} \sum_{j=0}^{i-1} \norm{w_{t-j}-w_{t-j-1}} + GD\brb{\msize_t-\msize_t^\ab} \\
        &\le G\sum_{i=1}^{\msize_t^\ab} (\msize_t^\ab-i+1)\norm{w_{t-i+1}-w_{t-i}} + GD\brb{\msize_t-\msize_t^\ab}
    \end{align*}
    and thus, summing over $t \in [a,b]$, we have
    \begin{align*}
        &\sum_{t=a}^b \abs*{f_t(w_{t-\msize_t:t}) - \hat f_t(\wt)} \\
        &\quad\le G\sum_{t=a}^b \sum_{i=1}^{\msize_t^\ab} (\msize_t^\ab-i+1)\norm{w_{t-i+1}-w_{t-i}} + GD\sum_{t=a}^b \brb{\msize_t-\msize_t^\ab} \\
        &\quad= G\sum_{t=a+1}^b \lrb{\sum_{s=t}^b \mathbb{I}\bcb{t \ge s-\msize_s^\ab+1}\brb{t-\brb{s-\msize_s^\ab}}}\norm{\wt-\wtmm} + GD\sum_{t=a}^b \brb{\msize_t-\msize_t^\ab} \\
        &\quad= G\sum_{t=a+1}^b \xi_t^\ab\norm{\wt-\wtmm} + GD\sum_{t=a}^b \brb{\msize_t-\msize_t^\ab} \\
        &\quad\le G\sum_{t=a+1}^b \xi_t^\ab\norm{\wt-\wtmm} + GDB_\ab^2 \;,
    \end{align*}
    where the last step follows by \Cref{lem:clipped-memory-error}.
    
    Using similar ideas as in the proof of \citet[Lemma~5.6]{qiu2026parameterfree}, we can also show that, for each $t \in [a,b]$,
    \begin{align}
        \xi_t^\ab &= \sum_{s=t}^b \mathbb{I}\bcb{t \ge s-\msize_s^\ab+1}\brb{t-\brb{s-\msize_s^\ab}}
        = \sum_{s=t}^b \bsb{t-\brb{s-\msize_s^\ab}}_+ \nonumber\\
        &\le \sum_{s=t}^b \lsb{\bmax^\ab - (s-t)}_+
        \le \sum_{i=0}^{\bmax^\ab-1} \brb{\bmax^\ab-i}
        = \frac{\bmax^\ab \brb{\bmax^\ab+1}}{2}
        \le \brb{\bmax^\ab}^2 \;, \label{eq:clipped-xi-bound}
    \end{align}
    where $[\cdot]_+ = \max\{\cdot, 0\}$ denotes the positive part of any given real number, and the last inequality holds because $\bmax^\ab$ is a non-negative integer.
    
    Combining the above inequalities then yields
    \begin{align*}
        \sum_{t=a}^b \abs*{f_t(w_{t-\msize_t:t}) - \hat f_t(\wt)}
        &\le G\sum_{t=a+1}^b \xi_t^\ab\norm{\wt-\wtmm} + GDB_\ab^2 \\
        &\le G\sum_{t=a+1}^b \xi_t\norm{\wt-\wtmm} + GDB_\ab^2 \;,
    \end{align*}
    where the last inequality holds because $\msize_t^\ab \le \msize_t$ for all $t \in [a,b]$.
    Via a similar derivation, we additionally have that
    \begin{align*}
        \sum_{t=a}^b \abs*{f_t(\cmp_{t-\msize_t:t}) - \hat f_t(\cmp_t)}
        &\le G\sum_{t=a+1}^b \xi_t^\ab\norm{\cmp_t-\cmp_{t-1}} + GDB_\ab^2 \\
        &\le GP_\ab\brb{\bmax^\ab}^2 + GDB_\ab^2 \;.
    \end{align*}

    Finally, combining the various inequalities shown so far, we can derive that
    \begin{align*}
        \RM_\ab(\cmp_{1:T})
        &= \sum_{t=a}^b \brb{\hat f_t(\wt) - \hat f_t(\cmp_t)} + \sum_{t=a}^b \brb{f_t(w_{t-\msize_t:t}) - \hat f_t(\wt)} - \sum_{t=a}^b \brb{f_t(\cmp_{t-\msize_t:t}) - \hat f_t(\cmp_t)} \\
        &\le \sum_{t=a}^b \inner{h_t, \wt-\cmp_t} + \sum_{t=a}^b \abs*{f_t(w_{t-\msize_t:t}) - \hat f_t(\wt)} + \sum_{t=a}^b \abs*{f_t(\cmp_{t-\msize_t:t}) - \hat f_t(\cmp_t)} \\
        &\le \sum_{t=a}^b \inner{h_t, \wt-\cmp_t} + G\sum_{t=a+1}^b \xi_t\norm{\wt-\wtmm} + GP_\ab\brb{\bmax^\ab}^2 + 2GDB_\ab^2 \;,
    \end{align*}
    which concludes the proof. 
\end{proof}

We can finally leverage this technical result to prove the main strongly adaptive regret guarantee for the setting of OCO with time-varying memory.

\MainAnyABMemory*
\begin{proof}[Proof of \Cref{thm:main-memory}]
    As the initial step, we apply \Cref{lemma:memory-redux}.
    For any interval $\ab \subseteq [T]$, this gives
    \begin{align}
        \RM_\ab(\cmp_{1:T})
        \le GP_\ab\brb{\bmax^\ab}^2 + 2GDB_\ab^2
        + \underbrace{\sum_{t\in\ab} \inner{h_t,\wt-\cmp_t} + \sum_{t=a+1}^b \lambda_t\norm{\wt-\wtmm}}_{\triangleq\tilde R_{\ab}(u_{a:b},\lambda_{a:b})} \;,\label{eq:initial-mem-bound}
    \end{align}
    where we define the movement coefficients as $\lambda_t \defeq G\xi_t\le GB^2$, whereas $h_t = \nabla\hat f_t(\wt)$ satisfies $\norm{h_t}\le H$ by \Cref{asm:unary-properties}.
    Observe that $\tilde R_\ab(u_{a:b},\lambda_{a:b})$ is the regret over $\ab$ for the time-varying movement-costs problem with linear losses given by $h_t$ and movement coefficients given by $\lambda_t$.
    Applying \Cref{thm:main} to this time-varying movement-costs problem yields
    \begin{align}
        \tilde R_\ab(u_{a:b},\lambda_{a:b})
        &=\hat\cO\Biggl(
            (D+P_\ab)\Log{b+1}^2(H+GB^2)
        \notag\\
            &\qquad\quad+ \sqrt{(D^2+DP_\ab) \sum_{t\in\ab} \left(\norm{h_t}^2 + \lambda_t\norm{h_t}\right)} \Log{b+1}
        \Biggr) \;,
        \label{eq:main-delay-before-xi-sum}
    \end{align}

    It remains to bound the cumulative interval-local movement coefficients. On one hand, 
    observe that for $\lambda_t=G\xi_t\le GB^2$ and $\norm{h_t}\le H \le G(B+1)$ (by 
    \Cref{asm:coordwise-lipschitz}),
    we can always bound 
    \begin{align}
    \sumab \norm{h_t}^2 + \lambda_t\norm{h_t} \le  \underbrace{(H+GB^2)\sumab\norm{h_t}}_{\eqqcolon V_\ab^{(1)}}
   \label{eq:basic-ht-bound}
    \end{align}
    On the other hand, we can alternatively bound
    \begin{align}
        \sumab\norm{h_t}^2 + \lambda_t\norm{h_t}\le \sumab\norm{h_t}^2 + GH\sumab\xi_t.\label{eq:adaptive-ht-bound}
    \end{align}
    By the definition of $\xi_t$, and using similar calculations as in the proof of \citet[Theorem~5.7]{qiu2026parameterfree}, we have
    \begin{align*}
\sum_{t=a}^{b}\xi_t
&=
\sum_{t=a}^{b}\sum_{s=t}^{T}
\mathbb{I}\{t\geq s-b_s+1\}
\bigl(t-(s-b_s)\bigr)  \\
&=
\sum_{s=a}^{T}
\sum_{t=a}^{\min\{b,s\}}
\mathbb{I}\{t\geq s-b_s+1\}
\bigl(t-(s-b_s)\bigr)  \\
&=
\sum_{s=a}^{T}
\sum_{t=\max\{a,s-b_s+1\}}^{\min\{b,s\}}
\bigl(t-(s-b_s)\bigr),
\end{align*}
where empty sums are treated as zero. We split the last expression into the
terms with $s\leq b$ and $s>b$. For $a\leq s\leq b$, we have
$\min\{b,s\}=s$, and hence
\begin{align}
\sum_{t=\max\{a,s-b_s+1\}}^{s}
\bigl(t-(s-b_s)\bigr)\leq
\sum_{t=s-b_s+1}^{s}
\bigl(t-(s-b_s)\bigr) =
\frac{b_s(b_s+1)}{2}
\leq b_s^2 .\label{eq:Vab1}
\end{align}
For $s>b$, we have $\min\{b,s\}=b$. The corresponding sum is nonzero only if
$s-b_s+1\leq b$, or equivalently $b_s>s-b$. In this case, the summand ranges
over at most the values $1,\ldots,b_s-(s-b)$. Therefore,
\begin{align*}
\sum_{t=\max\{a,s-b_s+1\}}^{b}
\bigl(t-(s-b_s)\bigr)
\leq
\sum_{i=1}^{[b_s-(s-b)]_+} i \leq
\bigl[b_s-(s-b)\bigr]_+^2 .
\end{align*}
Combining the two bounds gives
\begin{align}\label{eqn:movement-sum-memory}
\sum_{t=a}^{b}\xi_t
&\leq
\sum_{s=a}^{b} b_s^2
+
\sum_{s=b+1}^{T}
\bigl[b_s-(s-b)\bigr]_+^2 .
\end{align}
    To further bound the two terms in \Cref{eqn:movement-sum-memory}, direct calculation shows that
    \begin{align*}
        \sum_{t\in\ab} \msize_t^2
        &\le 2\sum_{t\in\ab} \brb{\msize_t^\ab}^2 + 2\sum_{t\in\ab} \brb{\msize_t - \msize_t^\ab}^2 \\
        &\le 2\sum_{t\in\ab} \brb{\msize_t^\ab}^2 + 2\lrb{\sum_{t\in\ab} \brb{\msize_t - \msize_t^\ab}}^2 \\
        &\le 2\sum_{t\in\ab} \brb{\msize_t^\ab}^2 + 2B_\ab^4 \;,
    \end{align*}
    where we used \Cref{lem:clipped-memory-error} at the last step. For the second term in \Cref{eqn:movement-sum-memory},
    \begin{align*}
        \sum_{s=b+1}^{T} \bigl[b_s - (s-b)\bigr]_+^2 
        &\le \max_{s \in [b+1, T]} \bigl[b_s - (s-b)\bigr]_+ \cdot \sum_{s=b+1}^{T} \bigl[b_s - (s-b)\bigr]_+ \\
        &\le B_{[b+1, T]} \sum_{s=b+1}^{T} \bigl[b_s - (s-b)\bigr]_+ \\
        &\le B_{[b+1, T]} \cdot B_{[b+1, T]}^2 \\
        &=B_{[b+1, T]}^3 \;.
    \end{align*}
    relying once more on \Cref{lem:clipped-xi-error}.
    Hence, we can bound
    \begin{align*}
        \sumab\lambda_t\norm{h_t}
        &\le HG\sum_{t\in \ab} \xi_t \\
        &\le HG\Bigg(\sum_{t\in \ab} \msize_t^2+B_{[b+1, T]}^3\Bigg) \\
        &\le2HG\sum_{t\in \ab} \brb{\msize_t^\ab}^2 + 2HG B_\ab^4 +HGB_{[b+1,T]}^3\;\\
        &\le 3HGB^4+2HG\sum_{t\in \ab} \brb{\msize_t^\ab}^2.
    \end{align*}
    We can also naively bound 
    \begin{align*}
        \sumab\lambda_t\norm{h_t}\le HG \sumab\xi_t\le HGB^2\abs{b-a+1}.
    \end{align*}
    Hence, 
    \begin{align}
        &\sumab\norm{h_t}^2+\lambda_t\norm{h_t}\nonumber\\
        &\qquad\le 
        \underbrace{
        \sumab\norm{h_t}^2
        +\min\Set{HGB^2\abs{b-a+1}, 3HGB^4 +2HG\sumab(b_t^\ab)^2}}_{\eqqcolon V_\ab^{(2)}}
        \label{eq:Vab2}
    \end{align}
    Therefore, combining \Cref{eq:basic-ht-bound,eq:Vab2}, we have $\sumab\norm{h_t}^2+\lambda_t\norm{h_t}\le \min\Set{V_\ab^{(1)},V_\ab^{(2)}}$.
    Plugging this back into \Cref{eq:main-delay-before-xi-sum} yields
    \begin{align*}
        \tilde R_\ab(\cmp_\atob,\lambda_\atob) &= \hat\cO\Bigg(
        (D+P_\ab)\log(b+1)^2(H+GB) \\
        &\qquad
        + \sqrt{(D^2+DP_\ab)\sbrac{V_\ab^{(1)}\minOp V_\ab^{(2)}}}\log(b+1),
        \Bigg)
    \end{align*}
    and likewise plugging this into \Cref{eq:initial-mem-bound} yields the stated guarantee:
    \begin{align*}
        \RM_\ab(\cmp_{1:T})
        &=
        \hat\cO\Bigg(
        (D+P_\ab)\sbrac{H+GB+GB^2}\log^2(b+1) \\
        &\qquad
        + \sqrt{(D^2+DP_\ab)\min\Set{V_\ab^{(1)},V_\ab^{(2)}}}\log(b+1)
        \Bigg)\\
        &=
        \hat\cO\Bigg(
        (D+P_\ab)\sbrac{H+GB^2}\log^2(b+1) \\
        &\qquad
        + \sqrt{(D^2+DP_\ab)\min\Set{V_\ab^{(1)},V_\ab^{(2)}}}\log(b+1)
        \Bigg)\qedhere
    \end{align*}
\end{proof}

%% file: sections/appendix-lemmas.tex
\section{Supporting Lemmas}

The following lemma follows from \citet[Theorem 7]{cutkosky2020parameter}, and is provided for completeness; we extract the key partition from their argument, the existence of which lets us improve $P_\ab\sqrt{b-a}$ dependencies to $\sqrt{P_\ab(b-a)}$ dependencies.
\begin{restatable}{lemma}{path-length-partition}\label{lemma:path-length-partition}
For any $\ab$ and comparator sequence $\cmp_a,\ldots,\cmp_b$ in $\ww$ with
$\sum_{t=a+1}^{b}\norm{\cmp_{t}-\cmp_{\tmm}}\ge D$, there is a disjoint set
of sub-intervals $I_1,\ldots,I_K$
such that the following hold:
\begin{enumerate}
    \item $\ab=\bigcup_{k=1}^K I_k$
    \item $P_{I_k}\le 2D,~\forall k\in[K]$ and $P_{I_{k}}\ge D$ for $k\in[K-1]$
    \item $K\le 1 + \frac{P_\ab}{D}$
\end{enumerate}
\end{restatable}
\begin{proof}
The proof follows the same argument as \citet[Theorem 7]{cutkosky2020parameter},  which constructs the partition using an iterative greedy construction, and is included for completeness. Let 
$I_k=[s_k,e_k]$ for any $k\in[K]$. 
Starting from $s_1=a$, let $e_1$ be the smallest 
index such that $P_{[a,e_1]}\ge D$. This implies that $P_{[a,e_1-1]}<D$. Since $\norm{\cmp_{e_1}-\cmp_{e_1-1}}\le D$,
we have $D\le P_{[a,e_1]}\le 2D$.
If $e_1<b$, let $s_2=e_1+1$ and let $e_2$ be the first round such that $P_{[s_2,e_2]}\ge D$ (If no such $e_2$ exists then we may set $e_2=b$).
Continuing on this way until we hit the right end-point of the interval ensures that 
for all $k\in[K-1]$ we have 
$D\le P_{I_k}\le 2D$, and $P_{I_K}\le 2D$ for the last interval. 
Hence, we have
\begin{align*}
    P_\ab \ge \sum_{k\in[K]} P_{I_k}
    \ge D(K-1),
\end{align*}
and re-arranging, we have the last property:
\begin{align*}
    1+\frac{P_\ab}{D} \ge K.
\end{align*}
The greedy construction can also be easily seen to ensure that $\bigcup_{k\in[K]} I_k=\ab$.
\end{proof}

We borrow the following standard lemma due to \citet{daniely2015strongly}, which 
shows that the geometric covering intervals imply the existence of a partition which has a certain doubling property.
\begin{lemma}[{Adapted from \citealp[Lemma~1.2]{daniely2015strongly}}]
  \label{lemma:gc-partition}
For any interval $[r, s] \subseteq[T]$, it can be partitioned into two sequences of disjoint and consecutive intervals, denoted by $I_{-p}, \ldots, I_0 \in \cI$ and $I_1, \ldots, I_q \in \cI$, such that
\[
\left|I_{-i}\right| /\left|I_{-i+1}\right| \leq 1 / 2, \quad \forall i \geq 1
\]
and
\[
\left|I_i\right| /\left|I_{i-1}\right| \leq 1 / 2, \forall i \geq 2 \;.
\]
\end{lemma}

The following provides a standard inequality used to
bound the stability terms in the proof of \Cref{thm:sleeping-mvmt-costs}, and
borrowed from \citet{jacobsen2022parameter}.
\begin{lemma}[{\citealp[Lemma 9]{jacobsen2022parameter}}]\label{lemma:tuning-sum}
  Let $L>0$ and $V_{t}\ge 4L^{2}+\sum_{s=1}^{\tmm}\norm{\gs}^{2}$ where $\norm{\gt}\le L$
  for each $t\in[T]$. Then
  \begin{align*}
    \sumtT \frac{\norm{\gt}^{2}}{V_{t}\log^{2}(V_{t}/L^{2})} \le 2
  \end{align*}
\end{lemma}

%% file: checklist.tex
\section*{NeurIPS Paper Checklist}

\begin{enumerate}

\item {\bf Claims}
    \item[] Question: Do the main claims made in the abstract and introduction accurately reflect the paper's contributions and scope?
    \item[] Answer: \answerYes{} %
    \item[] Justification: Each of the claims in the abstract are supported by new regret guarantees for which we provide detailed proofs.

\item {\bf Limitations}
    \item[] Question: Does the paper discuss the limitations of the work performed by the authors?
    \item[] Answer: \answerYes{} %
    \item[] Justification: Limitations and open questions of our results are discussed in \Cref{sec:discussion}.

\item {\bf Theory assumptions and proofs}
    \item[] Question: For each theoretical result, does the paper provide the full set of assumptions and a complete (and correct) proof?
    \item[] Answer: \answerYes{} %
    \item[] Justification: The assumptions are stated explicitly in \Cref{sec:setting} or directly in the theorem statements. The part about loss functions with time-varying memory requires specialized assumptions, which are explicitly stated and referred to in the respective \Cref{sec:memory}.

    \item {\bf Experimental result reproducibility}
    \item[] Question: Does the paper fully disclose all the information needed to reproduce the main experimental results of the paper to the extent that it affects the main claims and/or conclusions of the paper (regardless of whether the code and data are provided or not)?
    \item[] Answer: \answerNA{} %
    \item[] Justification: This paper does not include experiments.

\item {\bf Open access to data and code}
    \item[] Question: Does the paper provide open access to the data and code, with sufficient instructions to faithfully reproduce the main experimental results, as described in supplemental material?
    \item[] Answer: \answerNA{} %
    \item[] Justification: This paper does not include experiments requiring code.

\item {\bf Experimental setting/details}
    \item[] Question: Does the paper specify all the training and test details (e.g., data splits, hyperparameters, how they were chosen, type of optimizer) necessary to understand the results?
    \item[] Answer: \answerNA{} %
    \item[] Justification: This paper does not include experiments.

\item {\bf Experiment statistical significance}
    \item[] Question: Does the paper report error bars suitably and correctly defined or other appropriate information about the statistical significance of the experiments?
    \item[] Answer: \answerNA{} %
    \item[] Justification: This paper does not include experiments.

\item {\bf Experiments compute resources}
    \item[] Question: For each experiment, does the paper provide sufficient information on the computer resources (type of compute workers, memory, time of execution) needed to reproduce the experiments?
    \item[] Answer: \answerNA{} %
    \item[] Justification: This paper does not include experiments.
    
\item {\bf Code of ethics}
    \item[] Question: Does the research conducted in the paper conform, in every respect, with the NeurIPS Code of Ethics \url{https://neurips.cc/public/EthicsGuidelines}?
    \item[] Answer: \answerYes{} %
    \item[] Justification: The authors have reviewed the NeurIPS Code of Ethics and confirmed that the research conducted in the paper conforms to it.

\item {\bf Broader impacts}
    \item[] Question: Does the paper discuss both potential positive societal impacts and negative societal impacts of the work performed?
    \item[] Answer: \answerNA{} %
    \item[] Justification: The paper is a foundational work concerning online optimization theory; we do not anticipate any direct negative societal consequences.
    
\item {\bf Safeguards}
    \item[] Question: Does the paper describe safeguards that have been put in place for responsible release of data or models that have a high risk for misuse (e.g., pre-trained language models, image generators, or scraped datasets)?
    \item[] Answer: \answerNA{} %
    \item[] Justification: 
    The paper poses no such risks.

\item {\bf Licenses for existing assets}
    \item[] Question: Are the creators or original owners of assets (e.g., code, data, models), used in the paper, properly credited and are the license and terms of use explicitly mentioned and properly respected?
    \item[] Answer: \answerNA{} %
    \item[] Justification: The paper does not use existing assets.

\item {\bf New assets}
    \item[] Question: Are new assets introduced in the paper well documented and is the documentation provided alongside the assets?
    \item[] Answer: \answerNA{} %
    \item[] Justification: The paper does not introduce new assets.

\item {\bf Crowdsourcing and research with human subjects}
    \item[] Question: For crowdsourcing experiments and research with human subjects, does the paper include the full text of instructions given to participants and screenshots, if applicable, as well as details about compensation (if any)? 
    \item[] Answer: \answerNA{} %
    \item[] Justification: 
    The paper does not involve crowdsourcing nor research with human subjects.

\item {\bf Institutional review board (IRB) approvals or equivalent for research with human subjects}
    \item[] Question: Does the paper describe potential risks incurred by study participants, whether such risks were disclosed to the subjects, and whether Institutional Review Board (IRB) approvals (or an equivalent approval/review based on the requirements of your country or institution) were obtained?
    \item[] Answer: \answerNA{} %
    \item[] Justification: The paper does not involve crowdsourcing nor research with human subjects.

\item {\bf Declaration of LLM usage}
    \item[] Question: Does the paper describe the usage of LLMs if it is an important, original, or non-standard component of the core methods in this research? Note that if the LLM is used only for writing, editing, or formatting purposes and does \emph{not} impact the core methodology, scientific rigor, or originality of the research, declaration is not required.
    \item[] Answer: \answerNA{} %
    \item[] Justification: The core development of this research does not involve LLMs as any important, original, or non-standard components.

\end{enumerate}